\documentclass[11pt]{article}
\usepackage[margin=1in]{geometry}
\usepackage{amsmath,amssymb,amsthm,mathtools,bm,dsfont}
\usepackage{algorithm,algorithmicx,algpseudocode}
\usepackage{booktabs}
\usepackage{authblk}
\usepackage{multirow}	
\usepackage{microtype}
\usepackage{upgreek}
\usepackage{lmodern}
\usepackage{natbib}
\usepackage[font=footnotesize, margin=1cm]{caption}
\bibpunct[, ]{(}{)}{,}{a}{}{,}%
\usepackage[hidelinks]{hyperref}
\hypersetup{
	colorlinks=false
}

\newcommand{\interior}[1]{\mathrm{int}(#1)}
\newcommand{\I}{\mathcal{I}}

\DeclareMathOperator*{\argmax}{arg\,max}
\DeclareMathOperator*{\argmin}{arg\,min}

\newcommand*\diff{\mathop{} \mathrm{d}}

\newcommand{\Payoff}{D}
\newcommand{\LMOp}{\mathtt{LMO}_1} 
\newcommand{\LMOmu}{\mathtt{LMO}_2} 
\newcommand{\supp}{\mathrm{Supp}} 
\newcommand{\clarke}{\partial_{\bm{p}}^{\circ}}

\newtheorem{theorem}{Theorem}
\newtheorem{proposition}{Proposition}
\newtheorem{lemma}{Lemma}

\newtheorem{definition}{Definition}
\newtheorem{assumption}{Assumption}
\newtheorem{remark}{Remark}
\newtheorem{example}{Example}

\title{Pure Exploration via Frank--Wolfe Self-Play}
\author[1]{Xinyu Liu}
\author[2]{Chao Qin}
\author[1]{Wei You}

\affil[1]{Department of Industrial Engineering and Decision Analytics, The Hong Kong University of Science and Technology, \url{xliufn@connect.ust.hk}, \url{weiyou@ust.hk}}
\affil[2]{Stanford Graduate School of Business, Stanford University, \url{chaoqin@stanford.edu}}

\begin{document}

\maketitle
\begin{abstract}
We study pure exploration in structured stochastic multi-armed bandits, aiming to efficiently identify the correct hypothesis from a finite set of alternatives.
For a broad class of tasks, asymptotic analyses reduce to a maximin optimization that admits a two-player zero-sum game interpretation between an experimenter and a skeptic: the experimenter allocates measurements to rule out alternatives while the skeptic proposes alternatives.
We reformulate the game by allowing the skeptic to adopt a mixed strategy, yielding a concave--convex saddle-point problem.
This viewpoint leads to Frank--Wolfe Self-Play (FWSP): a projection-free, regularization-free, tuning-free method whose one-hot updates on both sides match the bandit sampling paradigm.
However, structural constraints introduce sharp pathologies that complicate algorithm design and analysis: our linear-bandit case study exhibits nonunique optima, optimal designs with zero mass on the best arm, bilinear objectives, and nonsmoothness at the boundary.
We address these challenges via a differential-inclusion argument, proving convergence of the game value for best-arm identification in linear bandits.
Our analysis proceeds through a continuous-time limit: a differential inclusion with a Lyapunov function that decays exponentially, implying a vanishing duality gap and convergence to the optimal value.
Although Lyapunov analysis requires differentiability of the objective, which is not guaranteed on the boundary, we show that along continuous trajectories the algorithm steers away from pathological nonsmooth points and achieves uniform global convergence to the optimal game value.
We then embed the discrete-time updates into a perturbed flow and show that the discrete game value also converges.
Building on FWSP, we further propose a learning algorithm based on posterior sampling.
Numerical experiments demonstrate game-value convergence with a vanishing duality gap.
\end{abstract}

\section{Introduction}
A central goal in modern data-driven science is to efficiently identify the correct hypothesis from a finite set of alternatives.
Examples include identifying the best alternative (best-arm identification, BAI; \citealt{audibert2010best}); 
identifying all alternatives that exceed a given threshold (thresholding bandits; \citealt{lai2012efficient}); 
and identifying all Pareto-optimal alternatives under multiple objectives (Pareto set identification; \citealt{kone2023adaptive}). 
To the best of our knowledge, this line of work traces back to \citet{chernoff1959sequential}. In recent years, multiple disciplines have witnessed a surge of interest in characterizing fundamental performance limits and developing algorithms that attain them \citep{glynn2004large, garivier2016optimal, russo2020simple, kasy2021adaptive, qin2025dual}.

Interestingly, these asymptotic analyses often reduce to a maximin optimization that admits a two-player zero-sum game interpretation between an experimenter and a skeptic; the game's equilibrium characterizes both the fundamental limits and the structure of optimal algorithms \citep{qin2024optimizing}.
The experimenter aims to answer a question of interest about the true state of nature $\bm{\theta}\in\mathbb{R}^d$.
To this end, the experimenter gathers convincing evidence by selecting $\bm{p} \in \Delta_K$ that specifies the allocation of measurement effort across the $K$ alternatives, where $\Delta_K\subset\mathbb{R}^K$ denotes the probability simplex.
The skeptic, in turn, chooses an alternative state $\bm\vartheta \in \mathrm{Alt}(\bm{\theta})$, which represents an instance with a different correct answer, in an attempt to mislead the experimenter's evidence collection. 
Given strategies $(\bm{p},\bm{\vartheta})$, the experimenter's payoff is
\[
\Gamma(\bm{p},\bm{\vartheta}) \triangleq \sum_{i \in [K]} p_i \mathrm{KL}(P_{\bm{\theta},i}\| P_{\bm{\vartheta},i}),
\]
where $P_{\bm{\theta},i}$ and $P_{\bm{\vartheta},i}$ denote the observation laws of arm $i$ under $\bm\theta$ and $\bm\vartheta$, respectively, and $\mathrm{KL}(\cdot \| \cdot)$ is the Kullback--Leibler divergence.
The payoff measures the expected amount of discriminative information against the skeptic's alternative $\bm{\vartheta}$ weighted by the experimenter's allocation $\bm{p}$.
Because the interaction is zero-sum, the experimenter seeks to maximize and the skeptic to minimize this quantity, leading to the maximin formulation that underpins attainable performance limits for many pure-exploration problems:
\begin{align*}
    \max_{\bm{p} \in \Delta_K} \inf_{\bm{\vartheta} \in \mathrm{Alt}(\bm{\theta})} \Gamma(\bm{p},\bm{\vartheta}).
\end{align*}

As in \cite{wang2021fast}, the skeptic's decision space $\mathrm{Alt}(\bm{\theta})$ decomposes as the union of a finite set of $\mathrm{Alt}_x(\bm{\theta})$, where each index $x\in \mathcal{X}(\bm\theta)$ denotes a ``confusing scenario'' that could flip the correct answer. 
The set $\mathrm{Alt}_x(\bm{\theta})$ collects all instances in which that scenario occurs.
The maximin problem is equivalent to
\begin{align}\label{eq:maximin_pure_strategy}
    \max_{\bm{p} \in \Delta_K} \min_{x \in \mathcal{X}} \Payoff(\bm{p}, x), 
    \quad\text{where}\quad 
    \Payoff(\bm{p}, x)\triangleq\inf_{\bm{\vartheta} \in \mathrm{Alt}_x(\bm{\theta})} \Gamma(\bm{p},\bm{\vartheta}),
\end{align}
where $\Payoff(\bm{p}, x)$ quantifies the discriminative information the experimenter can obtain against the \textit{most challenging} alternative in $\mathrm{Alt}_x(\bm{\theta})$.
Although asymptotic in nature, this maximin formulation both motivates and guides the design of practical algorithms. 
For example, in best-arm identification, a confusing scenario arises when a suboptimal arm $x\neq I^*(\bm\theta)$ appears superior to the true best arm $I^*(\bm\theta)$. 
The well-known (top-two) Thompson sampling allocates measurement effort so as to gather equal evidence to rule out \textit{each} such scenario, thereby enabling it to quickly identify the best arm \citep{russo2020simple}.
This corresponds to the structural requirement of the optimal allocation, known as the \textit{(discriminative) information-balance} condition: in \eqref{eq:maximin_pure_strategy},
$\Payoff(\bm{p}, x) = \Payoff(\bm{p}, x')$ for all $x,x' \neq I^*(\bm\theta)$ \citep{russo2020simple,garivier2016optimal}.
Thus, the skeptic is agnostic among the scenarios, as each is equally hard.

However, this information balance property can fail to hold for structured bandits (such as linear bandits), suggesting that a naive modification of Thompson sampling may be inadequate. 
In fact, adapting optimism-based bandit algorithms (such as Thompson sampling) to such problems can be fundamentally inefficient, as they inherently prioritize exploring high-performing arms over efficiently uncovering the most informative ones.
For illustration, Section~\ref{sec:linear_case_study} presents a simple linear-bandit instance, where the optimal allocation assigns zero sampling to the highest-performing arm, even though the objective is to identify it.
Thompson Sampling and its variants nonetheless allocate substantial effort to this arm, thereby incurring markedly suboptimal performance.
This example highlights the cost of ignoring the problem's information structure---an effect also observed in regret minimization; see \citet{lattimore2017end}.

An optimization expert might ask: if the maximin formulation characterizes performance limits, why not treat it as a standard optimization problem and derive pure-exploration algorithms by solving it directly?
A natural first approach is to view \eqref{eq:maximin_pure_strategy} as a maximization over $\bm{p}$ alone---maximize
$\min_{x \in \mathcal X} \Payoff(\bm{p}, x)$ with respect to $\bm{p}$---and apply the Frank--Wolfe algorithm.
When the feasible set for the allocation vector is the simplex, the linear minimization oracle returns a vertex (a one-hot vector), which aligns naturally with the one-arm-at-a-time sampling paradigm of bandit algorithms. 
However, even for unstructured problems such as BAI, numerical evidence indicates that Frank--Wolfe fails to attain asymptotic optimality \citep{menard2019gradient,degenne2020gamification,wang2021fast}. 
A key reason is \textit{nonsmoothness}: each $\Payoff(\bm{p}, x)$ in \eqref{eq:maximin_pure_strategy} is concave, so the objective is concave but typically nondifferentiable in $\bm p$, violating the bounded-curvature assumption for classical Frank--Wolfe guarantees \citep{jaggi2013revisiting}.

The nonsmoothness that breaks vanilla Frank--Wolfe comes from the skeptic's exact best responses---given $\bm{p}$, the skeptic solves $\min_{x \in \mathcal X} \Payoff(\bm{p}, x)$. 
But what if they move simultaneously?
A natural question one may ask is whether the original maximin formulation in \eqref{eq:maximin_pure_strategy} admits a Nash equilibrium. 
Its existence is crucial: it anchors simultaneous-play algorithms and provides primal--dual optimality certificates. 
Interestingly, \citet{qin2024optimizing} show that even in the case of BAI, a strict minimax inequality holds. 
To guarantee the existence of a Nash equilibrium, we allow the skeptic to play a mixed strategy $\bm{\mu} = (\mu_x)_{x\in \mathcal{X}} \in \Delta_{|\mathcal{X}|}$ over scenarios $x\in\mathcal{X}$. 
Under a mixed-strategy pair $(\bm{p},\bm{\mu})$, the payoff function becomes
\begin{equation}
\label{eq:mixed_payoff}
F(\bm{p},\bm{\mu}) \triangleq \mathbb{E}_{x \sim \bm{\mu}}[ \Payoff(\bm{p}, x)] = \sum_{x \in \mathcal{X}} \mu_x \Payoff(\bm{p}, x).
\end{equation}
Since the simplices are compact and convex and $F$ is continuous and concave--convex, we have
\begin{equation}\label{eq:maximin_mixed_strategy}
    \max_{\bm{p}\in\Delta_{K}} \min_{\bm{\mu}\in\Delta_{|\mathcal{X}|}} F(\bm{p},\bm{\mu}) = \min_{\bm{\mu}\in\Delta_{|\mathcal{X}|}} \max_{\bm{p}\in\Delta_{K}} F(\bm{p},\bm{\mu})
\end{equation}
by Sion's minimax theorem.
A perhaps surprising consequence of the mixed-strategy view is:
\begin{center}
    \emph{Simple Frank--Wolfe updates on both sides drive the players to a Nash equilibrium.}
\end{center}
For payoff $F(\bm{p},\bm{\mu})$ with simplex constraints, each Frank--Wolfe subproblem admits a closed form. 
In particular, at iterate $n$, both players react greedily to the linearized objective at $(\bm{p}_n,\bm{\mu}_n)$.
Here, $\bm{p}_n$ denote the proportional allocation across arms at time $n$, and let $\bm{\mu}_n$ denote the allocation over scenarios, so, $1/(n+1)$ is a natural step size. 
FWSP uses one-hot updates:
\begin{equation}\label{eq:FWSP}
\begin{cases}
    \bm{\mu}_{n+1} \gets \tfrac{1}{n+1}(n\bm{\mu}_{n} + \bm{e}_{x}), &x \in \argmin_{x'\in\mathcal{X}} \bigl[\nabla_{\bm{\mu}} F(\bm{p}_{n},\bm{\mu}_{n})\bigr]_{x'}, \\
    \bm{p}_{n+1} \gets \tfrac{1}{n+1}(n\bm{p}_{n} + \bm{e}_{i}), &i \in \argmax_{i'\in[K]} \bigl[\nabla_{\bm{p}} F(\bm{p}_{n},\bm{\mu}_{n})\bigr]_{i'},
\end{cases}
\end{equation}
with ties broken arbitrarily.
Taken together, this yields a simple modification of Frank--Wolfe, which we call \textit{Frank--Wolfe Self-Play}: in each round, each player treats the opponent's current mixed strategy as fixed and takes a greedy Frank--Wolfe step with respect to the linearized payoff. 
Remarkably, under minimal regularity assumptions this idea extends to a broad class of pure-exploration tasks and bandit models.

\subsection{Contributions}

\paragraph{Principled formulation with the right objective.} We recast the original maximin problem as a saddle-point problem by allowing the skeptic to adopt a mixed strategy.
This enables a near-direct application of the Frank--Wolfe method. 
This new perspective bridges the bandit pure-exploration literature, game theory, and optimization, highlighting mutual benefits across these communities.

\paragraph{Frank--Wolfe Self-Play algorithm.} Building on this reformulation, we propose a minimal modification of Frank--Wolfe: each player fixes the opponent's empirical mixed strategy and takes a greedy step for the linearized payoff. 
The resulting procedure is naturally bandit-friendly and simpler than existing methods---requiring no parameter tuning, no projections, and no iterate tracking. 
The optimality of FWSP is intuitive: it directly enforces the necessary and sufficient KKT conditions of the saddle-point problem, driving the associated residuals to zero and ensuring convergence.
However, a rigorous proof is challenging due to nonsmoothness, as demonstrated in our linear-bandit case study.

\paragraph{A linear-bandit case study.} 
Using simple linear-bandit examples, we reveal striking differences between structured and unstructured settings.
In particular, we present instances with (i) optimal solutions that allocate zero probability to the best arm; (ii) nonunique optimal solutions; (iii) bilinear objective functions; and (iv) pathological boundary nondifferentiability.
These examples expose a peculiar structure that breaks popular methods and motivates new techniques. 
Crucially, we cannot, \emph{a priori}, rule out boundary equilibria that lie arbitrarily close to nonsmooth regions.
Fortunately, via a novel differential-inclusion analysis, we show that FWSP steers trajectories away from such nonsmooth regions and attains global convergence in game value.

\paragraph{Novel proof via differential inclusion.} 
We establish a connection between the convergence of FWSP and that of its limiting continuous-time differential inclusion (DI), whose convergence we prove via a Lyapunov analysis.
We then embed the discrete-time FWSP iterates into a piecewise-linear interpolant that follows a mildly perturbed DI, and prove that the iterates track this DI and converge. 
Unfortunately, these analyses apply only when the objective $F$ is differentiable along the solution trajectory, which cannot be guaranteed \emph{a priori} because nonsmooth points and equilibria can both occur on the boundary.
To address this, we handle nonsmoothness via Clarke generalized gradients.
Perhaps surprisingly, we show that along continuous trajectories the algorithm steers away from pathological nonsmooth points and achieves uniform global convergence to the optimal game value.
In linear bandits, smoothness holds whenever $\Payoff(\bm p,x)>0$. 
Thus a well-designed procedure should drive all $\Payoff(\bm p,x)$ strictly positive, which occurs if and only if the active arms span the full feature space.
Our Frank--Wolfe step does exactly this: gradient bounds concentrate sampling on precisely those arms needed to span the target subspace, attaining full span in uniformly bounded time---after which $F$ is smooth and the Lyapunov analysis applies. 
This obviates any need to know the optimal active set in advance and, crucially, requires no modifications to the basic FWSP algorithm for the proof.

\subsection{Related Work}
Much of the existing literature adopts a pure optimization viewpoint and tackles the maximin problem via online optimization. 
Notably, \citet{menard2019gradient} use subgradient ascent, and \citet{wang2021fast} propose a variant of the Frank--Wolfe algorithm that handles nonsmooth objectives via subgradient subspaces. 
These methods maintain an ambient solution generated by the optimization procedure and then track this solution to define the bandit policy. 
More importantly, without the aid of a mixed strategy on the skeptic's side, one must carefully maintain feasibility (e.g., via entropy regularization) and handle nonsmoothness (e.g., via elaborate subgradient-space constructions).

When the objective is \emph{bilinear}, the maximin problem reduces to a two-player zero-sum matrix game, and our algorithm coincides with fictitious play \citep{brown1951iterative} from evolutionary game theory. 
\citet{robinson1951iterative} proved convergence of fictitious play with rate $O\bigl(t^{-1/(p+q-2)}\bigr)$; see also \citet{shapiro1958note}.
In the optimization literature, \citet{gidel2017frank} introduced Saddle-Point Frank--Wolfe (SPFW) and proved an $O(t^{-1/2})$ rate in \emph{strongly concave--convex} games, while \citet{chen2024last} obtained the same rate for a \emph{smoothed variant} of SPFW in the setting of monotone variational inequalities. Our objective is neither bilinear nor strongly concave--convex, and it is moreover nonsmooth on the boundary of the simplex. Thus, existing convergence analyses do not apply, further motivating the new technical arguments we develop.

In terms of proof technique, our work is related to the ordinary differential equation (ODE) method for stochastic approximation \citep{ljung1977analysis}, and its extension to differential inclusions \citep{benaim2005stochastic}. 
In evolutionary game theory, fictitious play has been analyzed via its continuous-time analogue, best-response dynamics (BRD), using Lyapunov methods \citep{hofbauer1995stability,harris1998rate}, with \citet{hofbauer2006best} extending the analysis to concave--convex zero-sum games. 
To obtain a closed-form, one-hot update, we replace exact best responses with FW steps, yielding different continuous-time dynamics.
While stochastic approximation via ODE methods is well understood \citep{Kushner1997StochasticAA,borkar2008stochastic}, the picture is less complete for differential inclusions; 
the only available result \citep{Nguyen2021StochasticAW} assumes conditions that our algorithm does not satisfy.
\citet{bandyopadhyay2024optimal} appears to be the only work that analyzes pure exploration via the fluid/ODE method, and it focuses exclusively on BAI.
It remains unclear whether their approach extends to other pure-exploration objectives or to structured bandit models.

We provide a more comprehensive literature review in Appendix~\ref{app:literature}.

\paragraph{Notation.}
For any integer $n$, let $[n] \triangleq \{1,2,\dots,n\}$.
Let $\mathbb{R}_{\ge 0}^n = \{ \bm{x}\in\mathbb{R}^n : x_i \ge 0, \forall i\in[n]\}$ denote the nonnegative orthant, and let
$\Delta_n = \bigl\{ \bm{x}\in\mathbb{R}_{\ge 0}^n : \sum_{i=1}^n x_i = 1 \bigr\}$
denote the probability simplex in $\mathbb{R}^n$. For a set $A \subset \mathbb{R}^n$, let $\interior{A}$ denote its (relative) interior.
Throughout the paper, we equip simplices with the $\ell_1$-norm and write, for a point $\bm{x}$ and sets $A,B\subset\mathbb{R}^n$,
$d(\bm{x},A) = \inf_{a\in A}\|\bm{x}-a\|_1$ and $d(A,B) = \sup_{a\in A} d(a,B).$
We denote by $\mathbb{B}$ the unit ball in the corresponding metric space.
For an interval $I\subset\mathbb{R}$ and a set $A\subset\mathbb{R}^n$, let $\mathcal{AC}(I;A)$ denote the set of absolutely continuous functions mapping $I$ to $A$. We endow $\mathcal{AC}(I;A)$ with the uniform metric
$d\bigl(\bm{x}(\cdot), \bm{y}(\cdot)\bigr) \triangleq \sup_{t\in I}\bigl\|\bm{x}(t) - \bm{y}(t)\bigr\|_{1}$ for two absolutely continuous functions $\bm{x}(\cdot)$ and $\bm{y}(\cdot)$ defined on interval $[0,I]$.
For $\bm{a},\bm{b}\in\mathbb{R}^n$, write $\bm{a}\circ\bm{b}\in\mathbb{R}^n$ for their component-wise product. For $\bm{v}\in\mathbb{R}^n$ and a symmetric positive (semi)definite matrix $M\in\mathbb{R}^{n\times n}$, define the $M$-norm by
$\|\bm{v}\|_{M} \triangleq \sqrt{\bm{v}^\top M \bm{v}}.$

\paragraph{Organization.}
The rest of the paper is organized as follows.
In Section~\ref{sec:linear_case_study}, we present a case study on linear bandits to illustrate the challenges of structured pure exploration and the advantages of our approach.
We introduce problem formulation in Section~\ref{sec:problem_formulation}.
Section~\ref{sec:algorithm} introduces the FWSP algorithm in discrete and continuous time.
Section~\ref{sec:convergence_analysis} presents the main result, convergence of the discrete iterates in game value.
In Section~\ref{sec:learning}, we propose a learning algorithm that builds on FWSP and posterior sampling.
Section~\ref{sec:proof} collects postponed proofs of our results.

\section{A Case Study on Linear Bandits}\label{sec:linear_case_study}

The maximin problem becomes delicate in structured settings (e.g., linear bandits): the optimizer is often nonunique and may lie on a boundary face of the simplex, with some coordinates equal to zero. 
In particular, even the highest-mean arm may optimally receive zero allocation---feature correlations allow other arms to supply all discriminative information. 
This structural sparsity clashes with top-two-style methods \citep{russo2020simple,qin2017improving,shang2020fixed,jourdan2022top,you2023information}, which reserve a nontrivial fraction of samples for the incumbent best arm and can misallocate effort in such instances. 

To illustrate the point, consider linear bandits with unknown parameter vector $\bm{\theta}\in\mathbb{R}^{d}$ and design matrix $A \in \mathbb{R}^{K \times d}$ whose $i$-th row is the arm feature vector $\bm a_i^\top$, so the mean reward for arm $i$ is $m_i \triangleq \bm a_i^\top \bm\theta$.
Pulling arm $i$ yields a reward $Y \sim \mathcal{N}(m_i,\sigma_i^2)$; rewards are independent across pulls (and across arms).
We instantiate \eqref{eq:maximin_pure_strategy} for identifying the best arm in linear bandits.
Let $I^* = \argmax_{i \in [K]} m_i$ denote the (assumed unique) best arm.
For each $x \in [K]\backslash \{I^*\}$, the alternative set is
$\mathrm{Alt}_x(\bm{\theta}) = \bigl\{ \bm{\vartheta} \in \mathbb{R}^d : \bm{a}_{I^*}^{\top} \bm{\vartheta} < \bm{a}_{x}^{\top} \bm{\vartheta} \bigr\},$ 
and furthermore,
\[
\Payoff(\bm{p}, x) = \frac{(m_{I^*} - m_x)^2}{2   \| \bm{a}_{I^*} - \bm{a}_x \|_{V_{\bm{p}}^{-1}}^2}, \quad \text{where} \quad \bm{p} \in \Delta_K, V_{\bm{p}} \triangleq \sum_{i=1}^K p_i \sigma_i^{-2} \bm{a}_i \bm{a}_i^\top, \text{ and }\| \bm{v} \|_{M}^2 \triangleq \bm{v}^\top M \bm{v}.
\]

\begin{example}[Thompson sampling and naive top-two fail]\label{ex:TS_fail}
Consider linear bandits with unit variances, unknown parameter $\bm\theta = \bm e_1\in\mathbb{R}^2$, and three arms specified by $\bm a_1 = \bm e_1, \bm a_2 = \bm e_2,$ and $\bm a_3 = -2\bm{e}_1$.
We have $I^* = 1$, and $\mathcal{X} = \{2,3\}$. 
Note that $\Payoff(\bm{p},2) < \Payoff(\bm{p},3) $ for all $\bm{p}$.
Consequently, the maximin value in \eqref{eq:maximin_pure_strategy} is
$
    \frac{1}{2}\max_{\bm{p} \in \Delta_3} p_2(p_1 + 4p_3)/(p_1+p_2+4p_3). 
$
There is a unique optimal allocation $\bm{p}^* = (0,2/3, 1/3)$.
This example contrasts with the unstructured bandit setting: (i) no information balance condition holds as $\Payoff(\bm{p},2) < \Payoff(\bm{p},3)$, and (ii) the optimal allocation may assign zero mass to some arms, including the best arm.
Thompson sampling draws posterior samples and acts greedily with respect to the sampled values, whereas the top-two algorithm \citep{russo2020simple} always selects the empirical best arm as the top candidate.\footnote{Even when the top-two algorithm sets the optimal tuning parameter $\beta = p_1^*=0$, it fails, since its asymptotic allocation is $(0,1,0)$.}
For both algorithms, arm $1$ is selected too often, which is not ideal because arm $3$ is more informative than arm $1$ since both arms are collinear with $\bm\theta$ but $\bm{a}_3$ has a larger norm.
By contrast, FWSP samples greedily according to the linearized objective. 
Since $\Payoff(\bm{p},2) < \Payoff(\bm{p},3) $, $\bm{\mu}$ quickly converges to $\bm{\mu}^* = (\mu^*_2,\mu^*_3) = (1,0)$ and hence $\nabla_{\bm{p}}F(\bm{p},\bm{\mu})$ converges to $\nabla_{\bm{p}}\Payoff(\bm{p},2) = (p_2^2, (p_1+4p_3)^2, 4p_2^2)/[2(p_1 + p_2 + 4p_3)^2]$.
FWSP always samples the arm with the largest component of $\nabla_{\bm{p}}F$, which implies zero allocation\footnote{Since $p_2^2<4p_2^2$ for any $p_2 > 0$. Even if $p_2(0) = 0$, the term $(p_1+4p_3)^2$ is largest and hence drives $p_2$ positive.} to arm 1 and the balancing condition $(p_1+4p_3)^2 = 4p_2^2$, leading to the optimal solution $\bm{p}^* = (0, 2/3, 1/3)$.
\end{example}

\begin{example}[Nonunique optimal solutions]
    Consider linear bandits with unit variances, unknown parameter $\bm\theta = \bm e_1\in\mathbb{R}^2$, and three arms specified by $\bm a_1 = \bm e_1, \bm a_2 = \bm e_2,$ and $\bm a_3 = -\bm{e}_1$.
    We have $I^* = 1$. 
    Note that $\Payoff(\bm{p},2) < \Payoff(\bm{p},3) $ for all $\bm{p}$.
    Consequently, the maximin value in \eqref{eq:maximin_pure_strategy} becomes
    $
        \frac{1}{2}\max_{\bm{p} \in \Delta_3} p_2(p_1 + p_3)/(p_1+p_2+p_3). 
    $
    For any $z\in[0,1/2]$, the point $\bm{p} = (z,1/2, 1/2-z)$ is an optimal solution.
    The continuum of saddle points complicates optimization and often leads to nonconvergence, cycling, or instability. It likewise complicates the design and analysis of tracking-type algorithms in pure exploration.
\end{example}

\begin{example}[Bilinear objective]\label{rmk:bilinear}
    Consider linear bandits with unit variances, unknown parameter $\bm{\theta}=(1,1)^\top$, and three arms $\bm{a}_1 = -\bm{e}_1, \bm{a}_2 = -\bm{e}_2$ and $\bm{a}_3 = \bm{0}$.
    We have $I^* = 3$ and $F(\bm{p},\bm{\mu})=\tfrac{1}{2}(\mu_1 p_1+\mu_2 p_2).$
    This yields a maximin problem with a bilinear objective, which is considerably harder than the strictly concave--convex case. 
    In particular, rotational (cycling) dynamics are common---for example, with plain gradient descent-ascent in the unconstrained setting \citep{shugart2025negative} or with replicator dynamics under simplex constraints in \citet[Section~7.4]{hofbauer1998evolutionary}.
    Achieving convergence typically requires extragradient or optimistic-gradient variants, coupled with careful iterate averaging and step-size selection. 
    These issues make algorithm design for linear bandits especially challenging.
\end{example}

\begin{remark}[Nonsmoothness at the boundary]\label{rm:boundary_gradient}
The gradient $\nabla_{\bm{p}}F$ need not be continuous when
$\bm{p}$ lies on the boundary of the simplex $ \Delta_{K}$.
A simple illustration comes from Gaussian BAI, a special case of linear bandits with $\bm{a}_i = \bm{e}_i$ for $i\in [K]$ and $\bm{\theta}\in\mathbb{R}^K$. 
For any $x\neq I^{*}$, we have $\Payoff(\bm{p}, x)=
    \frac{(\theta_{I^{*}}-\theta_{x})^2}{2\sigma^{2}} 
    \frac{p_{I^{*}} p_{x}}{p_{I^{*}}+p_{x}}.$
Note that by definition $\Payoff(\bm{p}, x) = 0$ if $p_{I^{*}}=p_{x}=0$ and the function is continuous even at the corner
$p_{I^{*}}=p_{x}=0$, but its gradient is not.
Indeed, fixing any point $\bm{p}\in \Delta_{K}$ with
$p_{I^{*}}=p_{x}=0$, the partial derivatives with respect to $p_{I^{*}}$ and $p_{x}$ are zero.
However, the limiting gradient values along different approach directions disagree, so $\Payoff(\bm{p}, x)$ is Lipschitz but not continuously differentiable.
\end{remark}

\begin{example}[Frank--Wolfe fails]\label{ex:FW_fails}
    Consider Gaussian BAI, the pure-exploration linear-bandit setting with arm features $\bm{a}_i = \bm{e}_i$ for $i\in [3]$ and $\bm{\theta}\in\mathbb{R}^3$. 
    The vanilla Frank--Wolfe update takes an FW step using the gradient of the objective $\Payoff(\bm{p}, x)$ evaluated at any minimizer $x \in \argmin_{x\in\mathcal{X}} \Payoff(\bm{p}, x)$. 
    For the instance $\bm\theta = (1,0,0)$ with unit variance, the saddle point is $\bm{p}^* = \bigl(\sqrt{2}-1, 1-\tfrac{\sqrt{2}}{2}, 1-\tfrac{\sqrt{2}}{2}\bigr)$ and $\bm{\mu}^* = (0.5,0.5)$.
    Nevertheless, the vanilla Frank--Wolfe dynamics converge to the uniform allocation $\bm{p} = \bigl(1/3,1/3,1/3\bigr)$: FW steps on  $\Payoff(\bm{p}, x)$ dictates that $p_1 \approx p_x$ for any\footnote{This argument rests on information balance, i.e., $\Payoff(\bm{p}, 2)=\Payoff(\bm{p}, 3)$, and FW steers the iterates toward it: if $p_2 < p_3$, then $\Payoff(\bm{p}, 2)<\Payoff(\bm{p}, 3)$ and vice versa.} $x \in\{2,3\}$, due to the symmetry of $\Payoff(\bm{p}, x)$ in Remark~\ref{rm:boundary_gradient}.
\end{example}

\section{Problem Formulation}\label{sec:problem_formulation}

Consider a stochastic multi-armed bandit with $K$ arms indexed by $i \in [K]$.
Sampling arm $i$ results in a noisy observation of its performance, drawn from a distribution $P_{\bm{\theta},i}$ that depends on a fixed but unknown state of nature $\bm{\theta} \in \mathbb{R}^d$, and rewards across different arms are mutually independent.
The parameter $\bm{\theta}$ determines potential information share among arms \citep{degenne2020structure}; see Section~\ref{sec:linear_case_study} for the canonical linear-bandit case, where the means depend linearly on $\bm\theta$.

The experimenter's goal is to answer a question about the true state $\bm{\theta}$ by adaptively allocating the sampling budget across the arms and observing noisy rewards from the selected arms.
We assume the question has a \emph{unique} answer, denoted by $\I(\bm{\theta})$.
Canonical questions about the mean rewards include: identifying
(i) the unique best arm with the highest mean;
(ii) $k$ arms with the highest means; (iii) all arms with mean above a threshold \citep{lai2012efficient}; and (iv) identify all Pareto optimal arms, when multiple objectives are considered simultaneously (Pareto set identification; \citealt{kone2023adaptive}).

For a given instance $\bm{\theta}$, the intrinsic complexity of the fixed-confidence pure-exploration task is closely tied to the \emph{alternative set}, consisting of parameters that yield a different correct answer than $\I(\bm{\theta})$:
\[
\mathrm{Alt}(\bm{\theta}) \triangleq \{ \bm{\vartheta} \in \mathbb{R}^d : \I(\bm{\vartheta}) \neq \I(\bm{\theta}) \}.
\]
We impose the following mild assumption, which is satisfied by most common tasks.
\begin{assumption}[\citealt{wang2021fast}]\label{assumption:correct_answer}
The alternative set $\mathrm{Alt}(\bm{\theta})$ is a finite union of convex sets.
That is, there exists a finite index set $\mathcal{X}(\bm{\theta})$ and convex sets $\{\mathrm{Alt}_x(\bm{\theta}) : x \in \mathcal{X}(\bm{\theta})\}$ such that $\mathrm{Alt}(\bm{\theta}) = \bigcup_{x \in \mathcal{X}(\bm{\theta})} \mathrm{Alt}_x(\bm{\theta}).$
\end{assumption}
The set $\mathcal{X}(\bm{\theta})$ captures the fundamental types of scenarios under which an alternative instance $\bm{\vartheta}$ produces an answer different from $\I(\bm{\theta})$.

We consider the mixed-strategy reformulation \eqref{eq:maximin_mixed_strategy} of the two-player zero-sum game,
where the maximizer chooses an allocation $\bm{p} \in \Delta_{K}$ and the minimizer chooses a distribution (mixed strategy) $\bm{\mu} \in \Delta_{|\mathcal{X}|}$ over $\mathcal{X}$.
Equivalently, we seek the saddle point(s) of $F(\bm{p},\bm{\mu})$ defined in \eqref{eq:mixed_payoff}.
We consider only noise distributions that follow canonical one-parameter exponential families (e.g., Gaussian with known variance, Bernoulli, Poisson) parameterized by their means.

\begin{lemma}[Proposition~1, \citealt{wang2021fast}]\label{lem:properties_Gamma_x}
    Under Assumption~\ref{assumption:correct_answer} and for any $x \in \mathcal{X}$, the function $\Payoff(\cdot, x):\mathbb{R}_{\ge 0}^K\to\mathbb{R}$ satisfies: (i) For every $\bm{p}\in\mathbb{R}_{\ge 0}^K$, $\Payoff(\bm{p}, x)$ is nonnegative, continuous, concave in $\bm{p}$, nondecreasing in each $p_i$, and homogeneous of degree one in $\bm{p}$, i.e., $\Payoff(\lambda \bm{p}, x) = \lambda \Payoff(\bm{p}, x)$ for all $\lambda > 0$. (ii) For every $\bm{p}\in\mathbb{R}^K_{>0}$, there exists a unique minimizer $\bm{\vartheta}_x$ to the infimum in \eqref{eq:maximin_pure_strategy} and $\Payoff(\bm{p}, x)$ is twice continuously differentiable in $\bm{p}$. Furthermore, we have $\bigl[\nabla_{\bm{p}}\Payoff(\bm{p}, x)\bigr]_i = \mathrm{KL}(P_{\bm{\theta},i}\| P_{\bm{\vartheta}_x,i})$.
\end{lemma}

\begin{assumption}\label{assumption:Gamma_x}
    For all $x\in\mathcal{X}$ and differentiable points $\bm{p}$, the gradients $\nabla_{\bm{p}} \Payoff(\bm{p}, x)$ exist, and there exists $M_p = M_p(\bm{\theta}) <\infty$ such that $\|\nabla_{\bm{p}} \Payoff(\bm{p}, x)\|_{\infty} \le M_p$.
\end{assumption}
The assumption on uniformly bounded gradient can be verified using \citet[Lemma 1]{wang2021fast}.

A pair $(\bm{p}^*,\bm{\mu}^*)$ is a \emph{Nash equilibrium} if and only if
\[
    F(\bm{p},\bm{\mu}^*) \le F(\bm{p}^*,\bm{\mu}^*) \le F(\bm{p}^*,\bm{\mu}),
\]
for all $\bm{\mu} \in \Delta_{|\mathcal{X}|}$ and $\bm{p} \in \Delta_{K}$.
Since $F(\bm{p},\bm{\mu})$ is concave in $\bm{p}$ and convex (linear) in $\bm{\mu}$ and the strategy sets are compact and convex, we can invoke minimax theorem for continuous concave--convex functions to show that Nash equilibrium exists. 
Define the unique optimal game value by $F^* \triangleq F(\bm{p}^*,\bm{\mu}^*).$

\begin{remark}[KKT conditions]\label{rmk:KKT}
    For our convex-concave game with simplex constraints, the Slater's condition holds, and the KKT system is necessary and sufficient for saddle-point optimality: $(\bm{p}^*,\bm{\mu}^*)$ is a Nash equilibrium if and only if:
    $\bm{p}^*\in\Delta_K$ and $ \bm{\mu}^*\in\Delta_{|\mathcal{X}|}$, and complementary slackness holds
    \begin{align*}
        \nabla_{\bm{\mu}} F(\bm{p}^*, \bm{\mu}^*) -F(\bm{p}^*,\bm{\mu}^*) \bm{1} & \ge \bm{0}, & \bm{\mu}^*\circ\bigl(\nabla_{\bm{\mu}} F(\bm{p}^*, \bm{\mu}^*) - F(\bm{p}^*,\bm{\mu}^*)\bm{1}\bigr) & =\bm{0},\\
        F(\bm{p}^*,\bm{\mu}^*) \bm{1} -\nabla_{\bm{p}} F(\bm{p}^*, \bm{\mu}^*) & \ge \bm{0}, & \bm{p}^*\circ\bigl(F(\bm{p}^*,\bm{\mu}^*)\bm{1} - \nabla_{\bm{p}} F(\bm{p}^*, \bm{\mu}^*)\bigr) & =\bm{0}.
    \end{align*}
    This implies that the skeptic mixes only among alternatives tied for the hardest, while the player plays only those arms tied for the highest marginal gain.
    These KKT conditions echo the design of our FWSP algorithm.
\end{remark}

However, there is no guarantee that the Nash equilibrium is unique in general. 
The structure of the motivating example in pure exploration further highlights the complications inherent in the general problem we consider. 
In particular, as observed in Example~\ref{rmk:bilinear} and Remark~\ref{rm:boundary_gradient}, the function $\Payoff(\bm{p}, x)$ is flat in most directions for unstructured BAI or linear bandits.
These flat directions indicate that the function is far from strictly concave, thereby complicating both the convergence analysis and the equilibrium study.

Since $F$ is a linear combination of $\Payoff(\bm{p}, x)$ functions, Lemma~\ref{lem:properties_Gamma_x} implies that $F$ is continuous on the entire domain (including the boundary); and for every fixed $\bm{\mu}\in \Delta_{|\mathcal{X}|}$, the function $F(\cdot, \bm{\mu})$ is concave in $\bm{p}$, for every fixed $\bm{p}\in \Delta_{K}$, the function $F(\bm{p}, \cdot)$ is convex in $\bm{\mu}$.
We further impose the following mild assumptions satisfied by our pure exploration maximin problems.
\begin{assumption}
  \label{assumption:regularity}
  The payoff function $F$ is twice continuously differentiable at any $(\bm{p},\bm{\mu}) \in \interior{\Delta_K} \times \Delta_{|\mathcal{X}|}$.
\end{assumption}

By Assumptions~\ref{assumption:Gamma_x} and~\ref{assumption:regularity}, for every $j \in \mathcal{X}$, the function $\Payoff(\bm{p}, j)$ is Lipschitz continuous (with respect to the $\ell_1$ norm) with constant $M_p$ on $\interior{\Delta_K} \times \Delta_{|\mathcal{X}|}$. 
Since $\Payoff$ is continuous on $\Delta_K$, the Lipschitz inequality extends to the closure, so $\Payoff(\bm{p}, j)$ is $M_p$-Lipschitz on $\Delta_K$. 
Consequently, $F$ is globally Lipschitz in the $\ell_1$ norm with a constant $L_F \triangleq M_p + M_{\mu}$, where $M_{\mu} \triangleq \sup_{\bm{p}\in\Delta_K} \max_{x \in \mathcal{X}} \Payoff(\bm{p},x) < \infty$.

Since $F$ may be nonsmooth in $\bm p$ on the boundary of the simplex, we adopt a generalized notion of derivative at such points.
Let $\mathcal{D}_F$ denote the set where $F$ is differentiable; Assumption~\ref{assumption:regularity} ensures $\interior{\Delta_K}\times\Delta_{|\mathcal{X}|}\subset\mathcal{D}_F$.
For $(\bm p,\bm\mu)\in(\Delta_K\times\Delta_{|\mathcal{X}|})\setminus\mathcal{D}_F$, the Clarke generalized gradient in $\bm p$ for a locally Lipschitz function is
\[
  \clarke F(\bm{p},\bm{\mu})
  \triangleq
  \operatorname{conv}\Bigl\{
     \lim_{n\to\infty}\nabla_{\bm{p}} F(\bm{p}^{n},\bm{\mu}^{n})
     : (\bm{p}^{n},\bm{\mu}^{n})\to(\bm{p},\bm{\mu}), (\bm{p}^{n},\bm{\mu}^{n}) \in \mathcal{D}_F
  \Bigr\},
\]
where $\mathrm{conv}\{\cdot\}$ denotes the convex hull.
If $F$ is $C^1$ at $(\bm{p},\bm{\mu})$, then $\clarke F(\bm{p},\bm{\mu})=\{\nabla_{\bm{p}} F(\bm{p},\bm{\mu})\}$.
Moreover, when $F(\cdot,\bm{\mu})$ is concave, \citet[Proposition~2.2.7]{clarke1990optimization} gives that $\clarke F(\bm{p},\bm{\mu})$ agrees with the concave superdifferential:
$
  \clarke F(\bm{p},\bm{\mu})
  =
  \bigl\{
     \bm f\in\mathbb R^{K} :
     F(\bm{p}',\bm{\mu})
        \le F(\bm{p},\bm{\mu})+\bm f^{ \top}(\bm{p}'-\bm{p}),
          \forall \bm{p}'\in \Delta_{K}
  \bigr\}.
$
We note that $F$ is smooth in $\bm\mu$ on $\Delta_{|\mathcal X|}$; thus the classical gradient suffices for the $\bm\mu$--updates.

\begin{lemma}[Bounds on Clarke generalized gradient]\label{lem:Clarke_UB}
    For every $\bm{p}\in\Delta_K$, every $x\in\mathcal{X}$ and every $\bm{g}\in\clarke \Payoff(\bm p,x)$, we have $\|\bm{g}\|_{\infty}\le M_{p}$. Consequently, for every $\bm{\mu}$ and every $\bm{f}\in\clarke F(\bm p,\bm{\mu})$, we have $\|\bm{f}\|_{\infty} \le M_{p}$.
\end{lemma}
\begin{proof}
By definition, any $\bm{g} \in \clarke \Payoff(\bm{p}, x)$ belongs to the convex hull of limit points of sequences $\nabla_{\bm{p}} \Payoff(\bm{p}_n, x)$, where $\Payoff(\bm{p}_n, x)$ is differentiable at $\bm{p}_n$ and $\bm{p}_n \to\bm{p}$.
By Assumption~\ref{assumption:Gamma_x}, $\| \nabla_{\bm{p}} \Payoff(\bm{p}_n, x) \|_\infty \le M_p$ for all $n$. 
The bound follows from the continuity of the $\ell_\infty$-norm.
\end{proof}
 
\section{Frank--Wolfe Self-Play in Discrete and Continuous Time}\label{sec:algorithm}

We consider dynamics in which each player fixes the opponent's empirical mixed strategy and takes a greedy Frank--Wolfe step for the \emph{linearized payoff}. 
At points where $F$ is differentiable, this reduces to choosing the coordinate with the largest partial derivative subject to the simplex constraint.
However, at points where $F$ is not classically differentiable, we replace the gradient by a Clarke subgradient.
\begin{definition}\label{def:LMO}
For $(\bm{p},\bm{\mu})\in \Delta_{K}\times\Delta_{|\mathcal X|}$, define the linear-minimization-oracle (LMO) correspondences\footnote{A correspondence is a set-valued function.}
\begin{align*}
    \LMOp(\bm{p},\bm{\mu})
        & =\mathrm{conv}\Bigl\{ \bm{q}\in\Delta_{K}:   \exists  \bm{f}\in\clarke  F(\bm{p},\bm{\mu}), \text{ s.t. } \bm{q} \in \argmax_{\bm{q}'\in\Delta_{K}} (\bm{q}')^\top \bm{f}\Bigr\},\\
    \LMOmu(\bm{p},\bm{\mu})            &=\argmin_{\bm{\nu}\in\Delta_{|\mathcal{X}|}} \bm{\nu}^\top \nabla_{\bm{\mu}}F(\bm{p},\bm{\mu}).
\end{align*}
\end{definition}
At points where $F$ is differentiable in $\bm p$, the Clarke generalized gradient coincides with the classical gradient, and the LMO reduces to $\LMOp(\bm{p}, \bm{\mu}) = \argmax_{\bm{q} \in \Delta_{K}} \bm{q}^{\top} \nabla_{\bm{p}} F(\bm{p}, \bm{\mu})$.
Our Frank--Wolfe Self-Play (FWSP) algorithm follows the discrete-time dynamics on $\Delta_{K} \times \Delta_{|\mathcal{X}|}$:
\begin{equation}\label{eq:discrete_dynamics}
\begin{aligned}
\bm{p}_{n+1} = \bm{p}_n + \tfrac{1}{n+1} \bigl( \bm{q}_n - \bm{p}_n \bigr), \quad &\text{with } \bm{q}_n \in \LMOp(\bm{p}_n,\bm{\mu}_n) \text{ and } \bm{p}_0 \in \Delta_K, \\ 
\bm{\mu}_{n+1} = \bm{\mu}_n + \tfrac{1}{n+1} \bigl( \bm{\nu}_n - \bm{\mu}_n \bigr), \quad &\text{with } \bm{\nu}_n \in \LMOmu(\bm{p}_n,\bm{\mu}_n)\text{ and } \bm{\mu}_0 \in \Delta_{|\mathcal X|}.
\end{aligned}
\end{equation}

\begin{remark}[Closed-form update rules]\label{rmk:closed_form}
    Since the linear oracle attains its optimum at a vertex, we may pick $\bm{q}_n = \bm{e}_{i_n}$ with $i_n = \argmax_{i \in [K]} \bigl[\nabla_{\bm{p}} F(\bm{p}_n,\bm{\mu}_n)\bigr]_i$
    thus $\bm{q}_n$ is a one-hot vector specifying the arm pulled at round $n$.
    Similarly, we may pick $\bm{\nu}_n = \bm{e}_{x_n}$ with $x_n = \argmin_{x\in\mathcal{X}} \bigl[\nabla_{\bm{\mu}} F(\bm{p}_n,\bm{\mu}_n)\bigr]_x = \argmin_{x\in\mathcal{X}} \Payoff(\bm{p}_n, x)$, 
    which matches the KKT detection rule in \cite{qin2025dual}.
    Together, this yields the pure-exploration \emph{bandit algorithm} in \eqref{eq:FWSP}, whose trajectory is a sample path of the discrete-time dynamics \eqref{eq:discrete_dynamics}.
    FWSP is computationally lightweight: it requires only gradient evaluations (and simple argmax/argmin over finite sets), and these gradients typically admit closed-form expressions; see \citet[Appendix~A]{qin2025dual}.
\end{remark}

\begin{remark}[Why does FWSP work?]
    Complementary slackness in Remark~\ref{rmk:KKT} stipulates that $p_i^* > 0$ only if $i \in \argmax_{i' \in [K]} \bigl[\nabla_{\bm{p}} F(\bm{p}^*,\bm{\mu}^*)\bigr]_{i'}$; likewise, $\mu_x^* > 0$ only if $x \in \argmin_{x' \in \mathcal{X}} \bigl[\nabla_{\bm{\mu}} F(\bm{p}^*,\bm{\mu}^*)\bigr]_{x'}$. 
    Thus, FWSP is consistent with complementary slackness: it samples an arm 
    \[i_n \in \argmax_{i \in [K]} \bigl[\nabla_{\bm{p}} F(\bm{p}_n,\bm{\mu}_n)\bigr]_{i}.\]
    Upon sampling $i_n$, the objective $F$ increases to first order (for sufficiently small step sizes), while $\bigl[\nabla_{\bm{p}} F\bigr]_{i_n}$ tends to decrease due to the concavity of $F$ in $\bm{p}$.
    Consequently, the KKT violation $\bigl([\nabla_{\bm{p}} F]_{i_n} - F\bigr)^{+}$ (which should be zero at optimality) tends to shrink.
    Conversely, for arms with $\bigl[\nabla_{\bm{p}} F\bigr]_{i} < F$, FWSP refrains from sampling them; if this ordering persists, the corresponding $p_i$ will decrease toward $0$, again respecting complementary slackness.
    A similar intuition applies on the skeptic's side.
    This suggests that the iterates move toward a strategy pair satisfying the full KKT conditions, i.e., a Nash equilibrium.
\end{remark}
We analyze the discrete-time dynamics through its limiting continuous-time dynamics, as the time index $n$ goes to infinity.
Specifically, we consider the following natural continuous version of \eqref{eq:discrete_dynamics}:
\begin{equation}\label{eq:continuous_dynamics}
    \begin{aligned}
      \frac{\diff}{\diff t} \bm{p}(t) 
        & \in \LMOp(\bm{p}(t), \bm{\mu}(t)) - \bm{p}(t),  & \bm{p}(0) &= \bm{p}_0 \in \Delta_{K},\\
      \frac{\diff}{\diff t} \bm{\mu}(t) 
        & \in \LMOmu(\bm{p}(t), \bm{\mu}(t)) - \bm{\mu}(t),   & \bm{\mu}(0) &= \bm{\mu}_0 \in \Delta_{|\mathcal X|},
    \end{aligned}
\end{equation}
This is a \emph{differential inclusion} (DI) as the LMO correspondences are not necessarily single-valued.

\section{Convergence Analysis}\label{sec:convergence_analysis}

Our main result is that the game value along the discrete-time dynamics \eqref{eq:discrete_dynamics} converges to the optimal value.

\begin{theorem}
\label{thm:discrete_convergence}
    Consider BAI in unstructured bandits with single-parameter exponential family distribution and BAI in Gaussian linear bandits.
    Under Assumptions~\ref{assumption:correct_answer}--\ref{assumption:regularity}, let $(\bm{p}_n,\bm{\mu}_n)_{n\ge 0}$ be the sequence generated by the discrete-time updates in \eqref{eq:discrete_dynamics} with initial condition $(\bm{p}_0,\bm{\mu}_0) \in \Delta_K \times \Delta_{|\mathcal{X}|}$. 
    Then $\lim_{n\to\infty}F(\bm{p}_n,\bm{\mu}_n) = F^*$, where $F^*$ is the optimal game value.
\end{theorem}
Our approach proceeds in two steps: (i) establish uniform convergence of the continuous-time DI (Section~\ref{sec:continuous_time_dynamics}); (ii) use perturbation analysis to show that the discrete iterates track DI trajectories (Section~\ref{sec:discrete_time_dynamics}).

\subsection{Convergence of the Continuous-Time Dynamics}\label{sec:continuous_time_dynamics}
We establish exponential convergence of the DI in \eqref{eq:continuous_dynamics} via a Lyapunov argument. 
To set the stage, we first collect basic properties of this DI. 
Longer proofs are deferred to Section~\ref{sec:proof}.

\begin{definition}[Upper hemicontinuity]
    A correspondence $F: X \rightrightarrows Y$ is upper hemicontinuous at a point $x \in X$ if for every open set $V \subseteq Y$ such that $F(x) \subseteq V$, there exists a neighborhood $U$ of $x$ satisfying $F(x') \subseteq V$ for all $x' \in U$.
\end{definition}

\begin{proposition}\label{prop:BR_properties}
    Under Assumptions~\ref{assumption:correct_answer} and~\ref{assumption:regularity}, for all $(\bm p,\bm\mu)\in \Delta_K\times\Delta_{|\mathcal X|}$, $\clarke F(\bm p,\bm\mu)$, $\LMOp(\bm p,\bm\mu)$ and $\LMOmu(\bm p,\bm\mu)$ are nonempty, compact, and convex.
    Furthermore, the correspondences $\clarke F$, $\LMOp$, and $\LMOmu$ are upper hemicontinuous.
\end{proposition}
Upper hemicontinuity of the right-hand side of the DI is crucial for establishing the existence of solutions.
Indeed, the next result establishes existence together with measurable selections to cast the DI as an ODE.
\begin{lemma}[Existence of solutions and measurable selections]\label{lem:existence_solution}
There exists an absolutely continuous solution $(\bm{p}(t),\bm{\mu}(t))$ defined on $[0,\infty)$ to \eqref{eq:continuous_dynamics} for any  $(\bm{p}_0,\bm{\mu}_0) \in \Delta_K \times \Delta_{|\mathcal{X}|}$. Furthermore, for any solution $(\bm{p}(t),\bm{\mu}(t))$ to \eqref{eq:continuous_dynamics}, there exist measurable selections $t\mapsto\bm{q}(t)\in\Delta_K$ and $t\mapsto\bm{\nu}(t)\in\Delta_{|\mathcal{X}|}$ such that
\[
\bm{q}(t) \in \LMOp(\bm{p}(t),\bm{\mu}(t)),\quad
\bm{\nu}(t) \in \LMOmu(\bm{p}(t),\bm{\mu}(t))\quad\text{and}\quad
\dot{\bm{p}}(t)=\bm{q}(t)-\bm{p}(t), \quad \dot{\bm{\mu}}(t)=\bm{\nu}(t)-\bm{\mu}(t).
\]
\end{lemma}
\begin{proof}
By Proposition~\ref{prop:BR_properties}, the correspondences $\LMOp$ and $\LMOmu$ are nonempty, convex, compact-valued, and upper hemicontinuous.
The existence of an absolutely continuous solution to the DI \eqref{eq:continuous_dynamics} therefore follows from \citet[Theorem~1, Section~4.2]{aubin_differential_1984}. Moreover, \citet[Corollary~1, Section~1.14]{aubin_differential_1984} ensures the existence of measurable selections along the trajectory.
\end{proof}

The following Lemma~\ref{lem:strict_positivity} shows that, although gradients may be ill-defined when $\bm p$ lies on the boundary of $\Delta_K$, any trajectory that starts in the interior under the continuous-time dynamics remains in the interior for every finite time.
The same holds for the discrete iterations in~\eqref{eq:discrete_dynamics}; hence the behavior of the correspondences \emph{on} the boundary is immaterial for our algorithm. 
This lets us analyze the dynamics without artificially repelling the iterates from the boundary, often enforced via forced exploration \citep{wang2021fast}. 
That said, the perturbation analysis of the discrete iterates still requires explicit treatment of boundary nonsmoothness.

\begin{lemma}\label{lem:strict_positivity}
For any solution $(\bm{p}(t),\bm{\mu}(t))$ to \eqref{eq:continuous_dynamics} with $(\bm{p}_0,\bm{\mu}_0) \in \interior{\Delta_K} \times \Delta_{|\mathcal{X}|}$ and any finite $t\ge 0$, we have $(\bm{p}(t),\bm{\mu}(t)) \in \interior{\Delta_K} \times \Delta_{|\mathcal{X}|}$.
\end{lemma}
\begin{proof}
By Lemma~\ref{lem:existence_solution}, the system is
$\dot{\bm{p}}=\bm{q}-\bm{p}$, $\dot{\bm{\mu}}=\bm{\nu}-\bm{\mu}$ for some measurable selections $\bm{q}(t)\in\Delta_K$, $\bm{\nu}(t)\in\Delta_{|\mathcal{X}|}$. It admits solution $\bm{p}(t)=e^{-t}\bm{p}_0+\int_0^t e^{-(t-s)}\bm{q}(s) \diff s.$
Each coordinate of $\bm{p}(t)$ satisfies $p_i(t)\ge e^{-t}p_{0,i}>0$, so $\bm{p}(t)\in\interior{\Delta_K}$ for all finite $t$. The simplex constraints are preserved since $\sum_i \dot p_i(t)=0$.
\end{proof}

Importantly, Lemma~\ref{lem:strict_positivity} does not rule out an \emph{asymptotic} approach to the boundary. 
This matters in structured bandits, where the optimal allocation $\bm p^*$ can be sparse and thus lie on the simplex boundary. 
Nevertheless, we show that although the trajectory remains in the strict interior for every finite time, the duality gap vanishes; consequently, every limit point achieves the optimal game value.

\subsubsection{(Almost) Global Exponential Stability}

The existence of solutions and measurable selections provided by Lemma~\ref{lem:existence_solution} ensures that the following candidate Lyapunov function is well defined on $\mathcal{D}_F$:
\[
    V(\bm{p},\bm{\mu})  = \max_{\bm{q}' \in \Delta_{K}} (\bm{q}' - \bm{p})^{\top} \nabla_{\bm{p}} F(\bm{p},\bm{\mu}) - \min_{\bm{\nu}' \in \Delta_{|\mathcal{X}|}} (\bm{\nu}' -\bm{\mu})^{\top} \nabla_{\bm{\mu}} F(\bm{p},\bm{\mu}). 
\]
By construction, $V(\bm{p},\bm{\mu}) \ge 0$. 
Since $(\bm{p}(t),\bm{\mu}(t))$ is absolutely continuous (Lemma~\ref{lem:existence_solution}) and, provided it evolves where $F$ is continuously differentiable, the composition
$V(t)\triangleq V(\bm{p}(t),\bm{\mu}(t))$ is absolutely continuous and thus differentiable for almost every $t\ge 0$.
The next theorem shows that $V$ is a valid Lyapunov function.
\begin{theorem} 
  \label{thm:convergence}
  Under Assumption~\ref{assumption:regularity} and suppose the solution satisfies $(\bm{p}(t),\bm{\mu}(t))\in\mathcal{D}_F$ for a.e. $t$, we have
  \[
    \frac{\diff}{\diff t} V(\bm{p}(t), \bm{\mu}(t)) \le -V(\bm{p}(t), \bm{\mu}(t)), \quad \text{for almost all } t \ge 0.
  \]
  Consequently, we have $V(t) \le V(0)e^{-t}$ by Gr\"{o}nwall's inequality.
  By Lemma~\ref{lem:Clarke_UB}, $V(0)$ is uniformly bounded, and the exponential decay is uniform in the initial condition.
\end{theorem}
In Theorem~\ref{thm:convergence}, we assume the trajectory remains in the differentiable region so that the Lyapunov function is well defined along the solution. By Lemma~\ref{lem:strict_positivity}, any interior initialization stays in the interior for all time, so the assumption holds even if the limit lies on the boundary; consequently, we obtain global uniform convergence.

\subsubsection{Convergence of the Game Value}
Since optimal solutions may be nonunique, the DI need not converge to a single Nash equilibrium.
Nevertheless, we show that the payoff along any solution trajectory converges to the optimal game value.
To this end, we define the duality gap
\[
    \mathrm{Gap}(\bm{p},\bm{\mu}) \triangleq 
    \max_{\bm{q}\in\Delta_{K}} F(\bm{q},\bm{\mu}) - \min_{\bm{\nu}\in\Delta_{|\mathcal{X}|}} F(\bm{p},\bm{\nu}).
\]   
By weak duality, $\mathrm{Gap}(\bm p,\bm\mu)\ge 0$.
The next result bounds the duality gap by the Lyapunov function.
\begin{lemma}\label{lem:gap}
Suppose $F$ is concave--convex, then the duality gap $\mathrm{Gap}(\bm{p},\bm{\mu})  \le  V(\bm{p},\bm{\mu})$ for all $(\bm{p},\bm{\mu}) \in \mathcal{D}_F$.
\end{lemma}

The following theorem establishes that the duality gap vanishes along the trajectory of the continuous-time dynamics and that the objective value converges to the optimal game value. 

\begin{theorem}[Convergence of the game value]\label{thm:value_convergence}
Let $(\bm{p}(t),\bm{\mu}(t))$ be any solution to the differential inclusion in~\eqref{eq:continuous_dynamics} with initial condition $(\bm{p}_0,\bm{\mu}_0) \in \mathcal{D}_F$.
Under Assumptions of Theorem~\ref{thm:convergence}, as $t\to\infty$,
\[
\mathrm{Gap}(t) \triangleq \mathrm{Gap}(\bm{p}(t),\bm{\mu}(t)) = 
\max_{\bm{q}\in\Delta_{K}} F(\bm{q},\bm{\mu}(t)) - \min_{\bm{\nu}\in\Delta_{|\mathcal{X}|}} F(\bm{p}(t),\bm{\nu}) \to 0.
\]
Consequently, $F(\bm{p}(t),\bm{\mu}(t)) \to F^*$, where $F^*$ denotes the optimal game value.
\end{theorem}

Let $\mathcal{E}$ denote the equilibrium set for the game. It is known from the literature that $\mathrm{Gap}(\bm{p}, \bm{\mu}) = 0$ is equivalent to $(\bm{p}, \bm{\mu}) \in \mathcal{E}$. Since $\mathrm{Gap}(\cdot)$ is continuous, any limit point of the trajectory $(\bm{p}(t), \bm{\mu}(t))$ must satisfy $\mathrm{Gap}(\cdot) = 0$ and hence belongs to $\mathcal{E}$.

\subsubsection{Global Exponential Stability in BAI for Unstructured and Linear Bandits}
The (almost) global exponential stability in Theorem~\ref{thm:convergence} and the game-value convergence in Theorem~\ref{thm:value_convergence} rely on the trajectory remaining in the region where $F$ is differentiable; by Lemma~\ref{lem:strict_positivity} and Assumption~\ref{assumption:regularity}, this holds for any interior initialization of the DI. 
To establish convergence of the discrete iterates in~\eqref{eq:discrete_dynamics}, however, we must analyze perturbations of the continuous-time DI that can produce boundary (and hence nonsmooth) evaluation points. 
Accordingly, we extend the convergence analysis to accommodate nonsmooth initial conditions for BAI in both unstructured and linear bandits.
\paragraph{BAI for unstructured bandits.}
Consider unstructured BAI in a one-parameter exponential family with a unique best arm $I^*$ and scenarios $\mathcal{X}=[K]\setminus\{I^*\}$. The payoff is
$
\Payoff(\bm p,x)
=\inf_{\lambda\in[\theta_x,\theta_{I^*}]}
\bigl\{p_{I^*} \mathrm{KL}(\theta_{I^*}\Vert \lambda)+p_x \mathrm{KL}(\theta_x\Vert \lambda)\bigr\}
$, for all $x\in\mathcal{X}$.
Define the zero-information scenario set $\mathcal Z(\bm p)\triangleq\{x\in\mathcal{X}:\Payoff(\bm p,x)=0\}=\{x: p_{I^*} p_x=0\}$ 
and constants
\[\kappa_x\triangleq \inf_{\lambda\in[\theta_x,\theta_{I^*}]}\max\{\mathrm{KL}(\theta_x\Vert\lambda), \mathrm{KL}(\theta_{I^*}\Vert\lambda)\}>0.\] 
Let $\kappa_{\min}\triangleq \min_{x\in\mathcal{X}}\kappa_x>0$.
To see why $\kappa_x > 0$, note that
$\inf_{\eta}\bigl\{\mathrm{KL}(\theta_{I^*}\|\eta)+\mathrm{KL}(\theta_x\|\eta)\bigr\}>0$.
For any $x \in \mathcal{X}$, if $p_{I^*}>0$ and $p_x>0$, let $\bar{\theta}_{x} = \frac{p_{I^{*}}\theta_{I^{*}} + p_x\theta_x}{p_{I^{*}} + p_x}$, then $\Payoff (\cdot,x)$ is differentiable and $\partial_{p_{I^{*}}}\Payoff(\bm p,x)=\mathrm{KL}(\theta_{I^*}\Vert\bar{\theta}_{x}),
\partial_{p_x}\Payoff(\bm p,x)=\mathrm{KL}(\theta_x\Vert\bar{\theta}_{x}),$
with all other coordinates zero.
As $p_x\downarrow 0$ with $p_{I^*}>0$ fixed,
$
\partial_{p_x}\Payoff(\bm p,x)\to \mathrm{KL}(\theta_x\Vert\theta_{I^{*}}), \partial_{p_{I^{*}}}\Payoff(\bm p,x)\to 0;
$
as $p_{I^{*}}\downarrow 0$ with $p_x>0$ fixed,
$
\partial_{p_{I^{*}}}\Payoff(\bm p,x)\to \mathrm{KL}(\theta_{I^{*}}\Vert\theta_x), \partial_{p_x}\Payoff(\bm p,x)\to 0,
$
so the Clarke generalized gradients at these edges are singletons.
At points with $p_{I^{*}} = p_x = 0$, the Clarke generalized gradient is a set:
\begin{align*}
\clarke \Payoff(\bm p,x) 
  = \operatorname{conv}\bigl\{g(\lambda): \lambda\in[\theta_x,\theta_{I^{*}}],\ &  g_{I^{*}}(\lambda)=\mathrm{KL}(\theta_{I^{*}}\Vert\lambda), g_x(\lambda)=\mathrm{KL}(\theta_x\Vert\lambda), \\
  & g_k(\lambda)=0, \forall k\notin\{I^{*},x\}\bigr\}.
\end{align*}

\begin{lemma}[Bounds on Clarke generalized gradient]\label{lem:Clarke_bound}
    For every $\bm{p}\in\Delta_K$, every $x\in\mathcal{X}$ and every $g\in\clarke \Payoff(\bm p,x)$, we have  $\max\{g_{I^{*}},g_x\} \ge \kappa_{x} \ge \kappa_{\min}$ and $g_k=0 \text{ for all } k\notin\{I^{*},x\}$ and $g_{k} \le M_{p}$ for all $k \in [K]$, where $M_{p}$ is defined in Lemma~\ref{lem:Clarke_UB}.
\end{lemma}
By Lemma~\ref{lem:strict_positivity}, the set $\mathcal Z(\bm p(t))$ is nonincreasing in $t$ along any solution. 
Let $0=\tau_0<\tau_1<\cdots$ be the times at which $\mathcal Z(\bm p(t))$ changes; on each \emph{phase} $[\tau_\ell,\tau_{\ell+1})$, $\mathcal Z(\bm p(t))$ is constant.
At any $t \in [\tau_\ell,\tau_{\ell+1})$ with $\mathcal Z(\bm p(t))\neq\varnothing$, $\LMOmu$ chooses among $\Delta_{\mathcal Z(\bm p(t))}$.
Define $S(t)\triangleq\sum_{j\notin\mathcal Z(\bm p(t))}\mu_j(t)$, then for any $\varepsilon\in(0,1)$, and any $t \in [\tau_\ell + T_\mu(\varepsilon),\tau_{\ell+1})$, we have $S(t)=S(\tau_\ell)e^{-(t-\tau_\ell)}\le e^{-(t-\tau_\ell)}\le \varepsilon$,
where $T_\mu(\varepsilon)\triangleq \log(1/\varepsilon)$.

We will show that the length $\tau_{\ell+1} - \tau_\ell$ of each phase must be uniformly finite.
The next two lemmas apply whenever a phase has $\mathcal Z(\bm p(t))\neq\varnothing$ and $S(t)$ has decayed sufficiently within that phase, i.e., after time $T_\mu(\varepsilon)$.
Let $\varepsilon_0\triangleq \min\{1/2, \kappa_{\min}/(4K M_{p})\}$ and $
T_0\triangleq T_\mu(\varepsilon_0)=\log (1/\varepsilon_0).$
\begin{lemma}[Zero-coordinate dominance]\label{lem:bai-dominance}
Fix a time $t$ with $\mathcal Z(\bm p(t))\neq\varnothing$ and $S(t)\le \varepsilon_0$. For every $\bm{f}(t)\in \clarke F(\bm p(t),\bm\mu(t))$, there exists $\bm{g}^{(j)}(t)\in\clarke \Payoff (\bm p(t),j)$ such that $\bm{f} = \sum_j \mu_j \bm{g}^{(j)}$, and 
\[
\max_{i: p_i(t)=0} f_i(t) \ge \frac{\kappa_{\min}}{2K},\qquad
\max_{i: p_i(t)>0} f_i(t) \le \varepsilon_0 M_{p} \le \frac{\kappa_{\min}}{4K}.
\]
Thus $\LMOp(\bm p(t),\bm\mu(t))$ is supported on $\{i: p_i(t)=0\}$.
\end{lemma}

\begin{proof}
By Lemma~\ref{lem:Clarke_bound}, summing $\bm{f}$ over $i\in\mathcal Z$ yields total mass $\ge (1-S(t))\kappa_{\min}\ge \frac12\kappa_{\min}$ on zero coordinates; averaging over at most $K$ gives the lower bound. On $\supp(\bm p(t))$, only $i\notin\mathcal Z$ contribute and each entry is $\le M_{p}$, giving the upper bound via $S(t)\le \varepsilon_0$.
\end{proof}

\begin{lemma}[Activation of zero coordinates]\label{lem:bai-activation}
Fix $t$ in a phase with $\mathcal Z(\bm p(t))\neq\varnothing$ and $S(t)\le \varepsilon_0$, and let $I_0=\{i: p_i(t)=0\}$. Then $\bm q(t)\in\LMOp(\bm p(t),\bm\mu(t))\subseteq \Delta_{I_0}$ and at $t+\log 2$ some $i\in I_0$ has $p_i\ge \tfrac{1}{2K}>0$.
\end{lemma}

\begin{proof}
Lemma~\ref{lem:bai-dominance} forces $\bm q(t)\in\Delta_{I_0}$; summing $\dot p_j=q_j-p_j$ over $I_0$ gives $
R_{I_0}(s)\triangleq \sum_{j\in I_0} p_j(s)$ satisfies $\dot R_{I_0}(s)=1-R_{I_0}(s), R_{I_0}(t)=0,$
so $R_{I_0}(s)=1-e^{-(s-t)}$ and $R_{I_0}(t+\log 2)> 1/2$.
\end{proof}

Applying Lemma~\ref{lem:bai-activation} at most $K-1$ times, we obtain that any solution enters the interior $\interior{\Delta_K}$ and thereafter remains in it. 
Moreover, this occurs by a globally uniform time $T^* \le (K-1)\bigl(T_0+\log 2\bigr)=(K-1)\bigl(\log(1/\varepsilon_0)+\log 2\bigr)$.
From time $T^{*}$ onwards, $F$ is differentiable along the trajectory, and Theorem~\ref{thm:value_convergence} yields uniform convergence of the game value.
\begin{theorem}[Finite-time interior entry and convergence]\label{thm:bai_convergence}
Consider exponential-family BAI. For any solution $\bm{p}(\cdot)$ of~\eqref{eq:continuous_dynamics} with any initial condition, $\bm p(t)\in\interior{\Delta_K}$ for all $t > T^{*}$. Consequently, $F$ is differentiable along the solution, and the game value converges uniformly.
\end{theorem}

\paragraph{BAI for linear bandits.}
We consider Gaussian linear bandits BAI in Section~\ref{sec:linear_case_study}. 
Define the active feature span $\mathcal{S}_{\bm p}\triangleq \mathrm{span}\{\bm{a}_i: p_i>0\}.$
Let $\mathcal{H}\triangleq \mathrm{span}\{\bm{a}_i: i\in[K]\}$ and interpret orthogonality and inverses on $\mathcal{H}$. Note that $\bm{\delta}_x \triangleq \bm{a}_{I^*} - \bm{a}_x \in \mathcal{H}$ for all $x\in \mathcal{X}$.
The scenario payoff equals $\Payoff(\bm p,x)=\inf_{\bm{\vartheta}: \bm{\delta}_x^\top\bm{\vartheta}\le 0}\frac12\|\bm{\vartheta}-\bm{\theta}\|_{V_{\bm p}}^2
=\frac{(m_{I^*} - m_x)^2}{2}\,\frac{1}{\bm{\delta}_x^\top V_{\bm p}^{-1} \bm{\delta}_x}$ if $\bm{\delta}_x\in \mathcal{S}_{\bm p},$ and $\Payoff(\bm p,x)=0$ if $\bm{\delta}_x\notin \mathcal{S}_{\bm p}$,
where $V^{-1}$ is the Moore--Penrose pseudoinverse.
Hence the zero-information scenario set is
$
\mathcal Z(\bm p)\triangleq\{x\in\mathcal{X}: \Payoff(\bm p,x)=0\}
=\{x: \bm{\delta}_x\notin \mathcal{S}_{\bm p}\}.
$
Moreover, $\Payoff(\bm p,x)$ is continuously differentiable whenever $\bm{\delta}_x\in \mathcal{S}_{\bm p}$.
Fix a face with span $\mathcal{S}\subseteq \mathcal{H}$ and let $O(\mathcal{S})\triangleq\{i\in[K]: \bm{a}_i\notin S\}$.
Define $
\delta_{\min}\triangleq \min_{x\in\mathcal{X}}(m_{I^*} - m_{x})>0$, $
\gamma(\mathcal{S})\triangleq \min_{\substack{\bm{u}\in \mathcal{S}^\perp\cap \mathcal{H}
\|\bm{u}\|=1}} \max_{j\in O(\mathcal{S})} \bigl(\bm{a}_j^\top \bm{u}\bigr)^2>0$, and $
\rho(\mathcal{S})\triangleq \max_{x\in\mathcal{X}}\bigl\|\mathrm{Proj}_{\mathcal{S}^\perp}\bm{\delta}_x\bigr\|^2<\infty$.
Strict positivity follows because $\mathcal{S}\oplus \mathrm{span}\{\bm{a}_j:j\in O(\mathcal{S})\}=\mathcal{H}$ implies $\gamma(\mathcal{S})>0$, while $\rho(\mathcal{S})<\infty$ since $\mathcal{X}$ is finite.
Finally, let $m(\mathcal{S})\triangleq \frac{\delta_{\min}^2}{2} \frac{\gamma(\mathcal{S})}{\rho(\mathcal{S})} > 0.$

\begin{lemma}[Bounds on Clarke generalized gradient]
\label{lem:lin-clarke-face}
Fix $\bm p$ with span $\mathcal{S}=\mathcal{S}_{\bm p}$ and $x\in\mathcal Z(\bm p)$.
Then every $g^{(x)}\in\clarke \Payoff(\bm p,x)$ admits a representation $g^{(x)}_i=\frac{(m_{I^*} - m_x)^2}{2} \bigl(\bm{a}_i^\top \bm{d}\bigr)^2$
for some $\bm{d}\in \mathcal{S}^\perp\cap \mathcal{H}$ with $\bm{\delta}_x^\top \bm{d}=1$.
In particular, $g^{(x)}_i=0$ whenever $\bm{a}_i\in \mathcal{S}$ and $\max_{j\in O(\mathcal{S})} g^{(x)}_j \ge m(\mathcal{S}).$
\end{lemma}

\begin{proof}
Take interior $\bm p^n\to \bm p$ with $\bm{\delta}_x\in \mathcal{S}_{\bm p^n}$ and write
$
[\nabla_{\bm p}\Payoff(\bm p^n,x)]_i
=\frac{(m_{I^*} - m_x)^2}{2}\frac{(\bm{a}_i^\top V_{\bm p^n}^{-1}\bm{\delta}_x)^2}{(\bm{\delta}_x^\top V_{\bm p^n}^{-1}\bm{\delta}_x)^2}.
$
Let $\bm{\zeta}_n\triangleq V_{\bm p^n}^{-1}\bm{\delta}_x/(\bm{\delta}_x^\top V_{\bm p^n}^{-1}\bm{\delta}_x)$.
Any accumulation point $\bm d$ of $(\bm{\zeta}_n)$ lies in $\mathcal{S}^\perp\cap \mathcal{H}$ and satisfies $\bm{\delta}_x^\top \bm d=1$, yielding the claimed representation.
For $\bm u=\bm d/\|\bm d\|$, $\max_{j\in O(\mathcal{S})}(\bm a_j^\top \bm d)^2=\|\bm d\|^2\max_{j\in O(\mathcal{S})}(\bm a_j^\top \bm u)^2\ge \|\bm d\|^2\gamma(\mathcal{S})$.
Since $\bm d\in \mathcal{S}^\perp$ and $\bm{\delta}_x^\top \bm d=1$, the minimal possible norm of such $\bm d$ is $1/\|\mathrm{Proj}_{\mathcal{S}^\perp}\bm{\delta}_x\|$, so $\|\bm d\|^2\ge 1/\|\mathrm{Proj}_{\mathcal{S}^\perp}\bm{\delta}_x\|^2$ and the final bound uses $m_{I^*} - m_x\ge \delta_{\min}$ and $\|\mathrm{Proj}_{\mathcal{S}^\perp}\bm{\delta}_x\|^2\le \rho(\mathcal{S})$.
\end{proof}

By Lemma~\ref{lem:strict_positivity}, along any solution, $\{i: p_i(t)=0\}$ is nonincreasing in $t$, hence $\mathcal{S}_{\bm p(t)}$ is nondecreasing and $\mathcal Z(\bm p(t))=\{x:\bm\delta_x\notin \mathcal{S}_{\bm p(t)}\}$ is nonincreasing.
Let $\{0=\tau_0<\tau_1<\cdots\}$ be the times when $\mathcal{S}_{\bm p(\cdot)}$ changes.
For $S(t)\triangleq \sum_{x\notin\mathcal Z(\bm p(t))}\mu_x(t)$ and any phase $[\tau_\ell,\tau_{\ell+1})$ with $\mathcal Z(\bm p(t))\neq\varnothing$, $\LMOmu$ chooses among $\Delta_{\mathcal Z(\bm p(t))}$, so $S(t)=S(\tau_\ell)e^{-(t-\tau_\ell)}$ a.e., and if $t - \tau_{\ell} > T_\mu(\varepsilon)$, then $S(t) < \varepsilon$ holds.

\begin{lemma}[Face-wise dominance]
\label{lem:lin-dominance}
Fix $t_0$ and let $\mathcal{S}=\mathcal{S}_{\bm p(t_0)}$, $O=O(\mathcal{S})$, and suppose $\mathcal Z(\bm p(t))\neq\varnothing$ on the current phase starting at $t_0$.
Set $\varepsilon_{\mathcal{S}}\triangleq \min\{1/2, m(\mathcal{S})/(4M_{p})\}$ and $
T_{\mathcal{S}}\triangleq \log(1/\varepsilon_{\mathcal{S}}).$
Then for every $t \ge t_0+T_{\mathcal{S}}$ within the same phase and every $\bm{f}(t)\in\clarke F(\bm p(t),\bm\mu(t))$, we have $\max_{j\in O} f_j(t) \ge \tfrac12 m(\mathcal{S})$ and $
\max_{i:\bm{a}_i\in \mathcal{S}} f_i(t) \le \tfrac14 m(\mathcal{S}).$
Hence $\LMOp(\bm p(t),\bm\mu(t))\subseteq \Delta_O$ on that time interval.
\end{lemma}

\begin{proof}
Decompose $\bm{f}=\sum_{x\in\mathcal Z}\mu_x \bm{g}^{(x)}+\sum_{y\notin\mathcal Z}\mu_y \bm{g}^{(y)}$.
For $j\in O$, Lemma~\ref{lem:lin-clarke-face} yields $\sum_{x\in\mathcal Z}\mu_x g^{(x)}_j\ge (1-S(t)) m(\mathcal{S})\ge (1-\varepsilon_S)m(\mathcal{S})\ge \tfrac12 m(\mathcal{S})$; the second sum is nonnegative and can be dropped in a lower bound.
When $\bm{a}_i\in \mathcal{S}$, $g^{(x)}_i=0$ for any $x\in\mathcal Z$ by Lemma~\ref{lem:lin-clarke-face}, and $\sum_{y\notin\mathcal Z}\mu_y g^{(y)}_i\le S(t) M_{p}\le \varepsilon_S M_{p}\le \tfrac12 m(\mathcal{S})$ by the choice of $\varepsilon_S$ and the phase-wise decay $S(t)\le \varepsilon_S$ for $t\ge t_0+T_{\mathcal{S}}$.
\end{proof}

\begin{lemma}[Span expansion]
\label{lem:lin-activation}
Under the conditions of Lemma~\ref{lem:lin-dominance}, let $t_1\triangleq t_0+T_{\mathcal{S}}$ and $R_O(t)\triangleq \sum_{j\in O} p_j(t)$.
Then by time $t_1+\log 2$ there exists $j\in O$ with $p_j(t)\ge \tfrac{1}{2|O|}>0$, so $\mathcal{S}_{\bm p(t)}$ strictly expands.
\end{lemma}

\begin{proof}
Lemma~\ref{lem:lin-dominance} forces $\bm q(t)\in\Delta_O$; summing $\dot p_j=q_j-p_j$ over $O$ gives, for a.e. $t\in[t_1,t_1+\log 2]$,
$\dot R_O(t)=1-R_O(t)$ with $R_O(t_1)=0$ and hence $R_O(t)=1-e^{-(t-t_1)}.$
\end{proof}

\begin{theorem}[Finite-time span expansion and convergence]
\label{thm:Linear_bandit_convergence}
From any $(\bm p_0,\bm\mu_0) \in \Delta_K \times \Delta_{|\mathcal{X}|}$, either $\mathcal Z(\bm p_0)=\varnothing$, or else there exists a strictly increasing chain of spans $
\mathcal{S}_0\subsetneq \mathcal{S}_1\subsetneq \cdots \subsetneq \mathcal{S}_r =  \mathcal{H}$ with $ r\le \operatorname{rank}\{\bm{\delta}_x: x\in\mathcal{X}\}\le \dim \mathcal{H},$
and times $0=\tau_0<\tau_1<\cdots<\tau_r$ such that on each epoch $[\tau_{\ell},\tau_{\ell+1})$: (i) $\mathcal{S}_{\bm p(t)}=\mathcal{S}_{\ell}$ and $\mathcal Z(\bm p(t))\neq\varnothing$; and (ii) within time at most $T_{\mathcal{S}_{\ell}}+\log 2$, the span expands to $S_{\ell+1}$.
Consequently, after total time $T^* \triangleq\sum_{\ell=0}^{r-1}(T_{\mathcal{S}_\ell}+\log 2)$ we reach $\mathcal{S}_{\bm p(t)} = \mathcal{H}$ and $\mathcal Z(\bm p(t))=\varnothing$.
From time $T^*$ onward, $F$ is differentiable along the solution, and the game value converges uniformly.
\end{theorem}

\subsection{Convergence of the Discrete-Time Algorithm}\label{sec:discrete_time_dynamics}
To prove Theorem~\ref{thm:discrete_convergence}, we view the discrete scheme as an asymptotic pseudotrajectory of the DI in~\eqref{eq:continuous_dynamics}. 
Specifically, we embed the iterates into continuous time via piecewise-linear interpolation, show that the interpolant solves a perturbed DI whose perturbation vanishes, and then use robustness of solutions to conclude that the perturbed paths track true solutions.

\begin{definition}[Linear interpolation]\label{def:interpolated DI}
Set $t_0=0$ and $t_{\ell}=\sum_{k=0}^{\ell-1}\frac{1}{k+1}$.
For $n\in\mathbb N$, we define 
\begin{equation}\label{def:hat_p_mu}
\hat{\bm{p}}(t_n+s)=\bm{p}_n+s \frac{\bm{p}_{n+1}-\bm{p}_n}{t_{n+1}-t_n},\quad
\hat{\bm{\mu}}(t_n+s)=\bm{\mu}_n+s \frac{\bm{\mu}_{n+1}-\bm{\mu}_n}{t_{n+1}-t_n},\quad s\in[0,t_{n+1}-t_n].
\end{equation}
Then $\hat{\bm{p}}(t_n)=\bm{p}_n$ and $\hat{\bm{\mu}}(t_n)=\bm{\mu}_n$ for all $n$.
\end{definition}

Let $G$ denote the right-hand side correspondence of the dynamics in~\eqref{eq:continuous_dynamics}:
\[
    G(\bm{p},\bm{\mu})\triangleq\big(\LMOp(\bm{p},\bm{\mu})-\bm{p}, \LMOmu(\bm{p},\bm{\mu})-\bm{\mu}\big).
\]
We consider an $\varepsilon$-perturbation that evaluates $G$ at nearby states and allows an $\varepsilon$-accurate approximation of the drift. 
This construction preserves upper hemicontinuity of the set-valued vector field, which is crucial for proving that perturbed solution trajectories remain close to true solutions.
\begin{definition}[$\varepsilon$-Perturbed DI]\label{def:epsilon_BR}
For $\varepsilon\ge0$ and $(\bm{p},\bm{\mu})\in \Delta_{K}\times\Delta_{|\mathcal X|}$, define
\[
    G^{\varepsilon}(\bm{p},\bm{\mu})
    = \Bigl\{
    \bm{y} \in \mathbb{R}^{K+|\mathcal{X}|} :
    \exists (\bm{p}',\bm{\mu}') \in \Delta_{K} \times \Delta_{|\mathcal{X}|} \text{ s.t. } 
    \|(\bm{p}, \bm{\mu}) - (\bm{p}', \bm{\mu}')\|_{1} 
    + d\bigl(
        \bm{y}, G(\bm{p}', \bm{\mu}')    
    \bigr) 
    \le \varepsilon
    \Bigr\}.
\]
The $\varepsilon$-perturbed DI is defined by
$\frac{\diff}{\diff t}\bigl(\bm{p}(t),\bm{\mu}(t)\bigr)\in G^{\varepsilon}\bigl(\bm{p}(t),\bm{\mu}(t)\bigr), 
(\bm{p}(0),\bm{\mu}(0))=(\bm{p}_0,\bm{\mu}_0).$
\end{definition}

We define $\mathcal T^{\varepsilon,T}$ as a solution correspondence that maps an initial condition in $\Delta_K\times\Delta_{|\mathcal X|}$ to the (nonempty) set of absolutely continuous solutions of the $\varepsilon$-perturbed DI on $[0,T]$.
The next result shows that the continuous-time interpolation introduced in Definition~\ref{def:interpolated DI} indeed evolves according to the perturbed DI. 

\begin{proposition}\label{prop:interpolation}
For any $n\in\mathbb N$ and $T\ge0$, the translated interpolation
\[
 \bigl(\hat{\bm p}(s+\cdot),\hat{\bm\mu}(s+\cdot)\bigr)\big|_{[0,T]}
 \in \mathcal T^{\frac{4}{n+1}, T}\bigl(\hat{\bm p}(s),\hat{\bm\mu}(s)\bigr).
\]
\end{proposition}

For a set-valued map $F: X \rightrightarrows Y$, the graph is defined as $\mathrm{graph}(F)\triangleq \bigl\{(x,y)\in X\times Y: y\in F(x)\bigr\}.$
The following proposition establishes the upper graphical convergence of the derivative and solution maps.
\begin{proposition}[Graph Stability]\label{prop:uniform_perturbation}
For any ${\varepsilon} > 0$,  $\mathrm{graph}(G^{\varepsilon})\subseteq \mathrm{graph}(G)+\varepsilon \mathbb{B}.$
For any $\delta>0$, $T>0$, there exists $\varepsilon= \varepsilon(\delta,T)>0$ such that
$
\mathrm{graph}(\mathcal{T}^{\varepsilon,T})\subseteq\mathrm{graph}(\mathcal{T}^{0,T})+\delta \mathbb{B}.
$
\end{proposition}
\begin{proof}
    The first part follows directly from the definition of the perturbation correspondence. The second part is a consequence of \citet[Proposition 1,  Chapter 2.2]{aubin_differential_1984}.
\end{proof}

The LMO solution maps are discontinuous when two gradient components are tied---an arbitrarily small gradient shift can flip the selected extreme point and produce $O(1)$ deviations over any fixed time horizon. 
Thus, the usual ODE method uniform perturbation stability with identical initial data fails for this discontinuous vector field. 
This is precisely why our $\varepsilon$-perturbation in Definition~\ref{def:epsilon_BR} allows nearby evaluation: it restores upper hemicontinuity and the robustness/closedness of solution sets needed for shadowing results.

\subsubsection{Proof of the Main Result, Theorem~\ref{thm:discrete_convergence}}\label{sec:discrete_convergence}
\begin{proof}[Proof of Theorem~\ref{thm:discrete_convergence}]
    Fix any convergence tolerance $\delta$.
    Lemma~\ref{lem:Clarke_UB} implies that $F$ is globally $L_F$-Lipschitz on $\Delta_K\times\Delta_{|\mathcal{X}|}$ in the $\ell_1$ norm.
    By Theorem~\ref{thm:bai_convergence} or~\ref{thm:Linear_bandit_convergence}, for any $\delta>0$ there exists $T_0(\delta)<\infty$ such that, for every initial condition in $\Delta_K\times\Delta_{|\mathcal{X}|}$ and every DI solution $(\bm{p}(t),\bm{\mu}(t))$, $|F(\bm{p}(t),\bm{\mu}(t))-F^*|\leq\delta/2$.
    Choose $n$ large enough that $4/(n+1)\le \varepsilon_0 \triangleq \varepsilon(\delta/(2L_F),T_0)$ from Proposition~\ref{prop:uniform_perturbation}.
    By Propositions~\ref{prop:interpolation} and~\ref{prop:uniform_perturbation}, for any $s\geq t_{n}$, we have
    $
    \bigl[(\hat{\bm{p}}(s),\hat{\bm{\mu}}(s)), \bigl(\hat{\bm{p}}(s+\cdot),\hat{\bm{\mu}}(s+\cdot)\bigr)\big|_{[0,T_0]}\bigr]
    \in \mathrm{graph}(\mathcal T^{\varepsilon_0,T_0})
    \subseteq \mathrm{graph}(\mathcal T^{0,T_0})+\frac{\delta}{2L_F}\mathbb{B}.
    $
    Therefore, for any $s\geq t_{n}$ there exists a DI solution $(\bm{p}^s(\cdot),\bm{\mu}^s(\cdot))\in\mathcal T^{0,T_0}(\bm{p}^s_0,\bm{\mu}^s_0)$ such that
    $
    \sup_{t\in[0,T_0]}\|(\bm{p}^s(t),\bm{\mu}^s(t))-(\hat{\bm{p}}(s+t),\hat{\bm{\mu}}(s+t))\|_1\le \frac{\delta}{2L_F},
    $
    which implies, by the Lipschitz property of $F$, that $
    \bigl|F(\hat{\bm{p}}(s+T_0),\hat{\bm{\mu}}(s+T_0))-F^*\bigr|\le \delta$.
    Since $\delta>0$ is arbitrary and $(\bm{p}_n,\bm{\mu}_n) = (\hat{\bm{p}}(t_{n}),\hat{\bm{\mu}}(t_{n}))$, we have $\lim_{n\to\infty} F(\bm{p}_n,\bm{\mu}_n)=F^*$.
\end{proof}

\section{Learning Algorithm}\label{sec:learning}

We propose a posterior-sampling-based algorithm that follows the spirit of FWSP. 
At each round $t \in \{0,1,\dots\}$, the experimenter selects an arm $i_t$ and observes a random reward $Y_{t+1, i_t}$ with mean $\mu_{i_t}$. 

\begin{assumption}[Posterior updating and sampling]
We assume access to a routine that, given new data $(i_t, Y_{t+1,i_t})$ and the current posterior $\Pi_t$ over $\bm{\theta}$, returns (i) the updated posterior $\Pi_{t+1}$ and (ii) a sampler generating draws from $\Pi_{t+1}$.
\end{assumption}
    In pure exploration for Gaussian linear bandits with a flat prior and known noise variances $\{\sigma_i^2\}_{i=1}^K$, an updater is immediate: Let $V_t \triangleq \sum_{i=1}^K n_{t,i}  \sigma_i^{-2}  \bm{a}_i \bm{a}_i^\top $ and $ \bm b_t \triangleq \sum_{i=1}^K n_{t,i}  \sigma_i^{-2}  \bm a_i   \bar y_{t,i},$
    where $n_{t,i}$ is the number of pulls of arm $i$ up to round $t$ and $\bar y_{t,i}$ is its sample mean. 
    If $V_t$ is full rank, the updated posterior is $\mathcal{N}(V_t^{-1}\bm b_t, V_t^{-1})$. These quantities admit online Sherman--Morrison updates upon each new observation.

Our learning algorithm samples $\widehat{\bm\theta}_t \sim \Pi_t$ and treats it as ground truth when computing the gradients of $F(\bm p,\bm \mu;\widehat{\bm\theta}_t)$, making the dependence on $\widehat{\bm\theta}_t$ explicit.
Rather than using the asymptotic-inspired Frank--Wolfe step for the skeptic, we adopt a posterior-sampling update in the spirit of top-two TS, see also \citet[Section~3.5.2]{qin2025dual}, which is often more effective at finite sample sizes.
Algorithm~\ref{alg:learn} summarizes the procedure; numerical results appear in Appendix~\ref{sec:num}. 
Theoretical analysis is left for future work.

\begin{algorithm}[hbtp]
    \caption{Posterior-sampling based FWSP}
    \label{alg:learn}
    \begin{algorithmic}[1]
        \renewcommand{\algorithmicrequire}{\textbf{Input:}}

        \Require{Uninformative prior $\Pi_0$; access to a posterior-updating oracle.}
        \For{$t = 0,1,\dots,T$}

        \State{Call the oracle to obtain the posterior $\Pi_{t}$ and sample  $\widehat{\bm{\theta}}_t \sim \Pi_t$.}

        \State{Repeatedly sample $\widetilde{\bm\theta}_t \sim \Pi_t$ until $\widetilde{\bm\theta}_t \in \mathrm{Alt}(\widehat{\bm\theta}_t)$.}
        
        \State{Choose any $x_t \in \{x' \in \mathcal {X}(\widehat{\bm\theta}_t) : \widetilde{\bm\theta}_t \in \mathrm{Alt}_{x'}(\widehat{\bm\theta}_t)\}$ and update $\bm{\mu}_{t+1} = \bm{\mu}_t + \tfrac{1}{t+1} \bigl( \bm{e}_{x_t} - \bm{\mu}_t \bigr)$.}
        
        \State{Choose any $i_t \in \argmax_{i \in [K]} \bigl[\nabla_{\bm{p}} F(\bm{p}_t,\bm{\mu}_{t+1};\widehat{\bm\theta}_t)\bigr]_i$ and update $\bm{p}_{t+1} = \bm{p}_t + \tfrac{1}{t+1} \bigl( \bm{e}_{i_t} - \bm{p}_t \bigr)$.}

        \State{Play arm $i_t$ and receive observation $Y_{t+1, i_t}$.}

        \EndFor
    \end{algorithmic}
\end{algorithm}

\section{Proofs}\label{sec:proof}
In this section, we present the deferred proofs of our results. These proofs rely on a number of classical facts about set-valued maps and differential inclusions, which we restate in Appendix~\ref{sec:facts} for ease of reference.

\begin{proof}[Proof of Proposition~\ref{prop:BR_properties}]
    Since $F(\cdot,\bm\mu)$ is concave, $\clarke F$ coincides with the concave superdifferential, is nonempty, convex, and closed, by \citet[Proposition~4.2.1]{bertsekas2003convex}. By Lemma~\ref{lem:Clarke_UB}, $\clarke F$ is also compact.
    To prove upper hemicontinuity, let $(\bm p^n,\bm\mu^n)\to(\bm p,\bm\mu)$ and $\bm f^n\in\clarke F(\bm p^n,\bm\mu^n)$ with $\bm f^n\to\bm f$. By \citet[Proposition~4.2.3]{bertsekas2003convex} for concave superdifferential, $\bm f\in\clarke F$.
    Thus the graph of $(\bm p,\bm\mu)\mapsto \clarke F(\bm p,\bm\mu)$ is closed. This is equivalent to upper hemicontinuity for compact-valued correspondence; see \citet[Theorem 3.2]{beavis1990optimization}.

    By Lemma~\ref{lem:properties_Gamma_x}, $\nabla_{\bm{\mu}}F$ are continuous on $\Delta_K\times\Delta_{|\mathcal{X}|}$, so the objectives in Definition~\ref{def:LMO} of $\LMOmu$ are continuous.
    For any $(\bm{p},\bm{\mu})$, $\LMOmu$ solve linear programs over simplices, so their argmin sets are nonempty, closed, and convex.
    Since the feasible sets are constant and compact, and the objectives are continuous, Berge's Maximum Theorem implies that $\LMOmu$ is upper hemicontinuous; see \citet[Theorem 3.6]{beavis1990optimization}. 
    
    For $\LMOp$, define 
    $Q=\bigcup_{\bm f\in\clarke F(\bm{p},\bm{\mu})}\argmax_{\bm q\in\Delta_K}\bm q^\top \bm f$.
    If $\bm q^n\in Q$ and $\bm q^n\to\bm q$, pick $\bm f^n\in\clarke F(\bm{p},\bm{\mu})$ with $\bm q^n\in\argmax_{\bm q}\bm q^\top \bm f^n$.
    By compactness of $\clarke F$, pass to a subsequence $\bm f^{n_k}\to\bm f\in\clarke F$. For any $\bm r\in\Delta_K$, $\bm r^\top \bm f = \lim_k \bm r^\top \bm f^{n_k}
    \le \liminf_k (\bm q^{n_k})^\top \bm f^{n_k}
    = \bm q^\top \bm f,$ so $\bm q\in\argmax \bm q^\top \bm f$, so $Q$ is closed, hence compact. Then $\LMOp=\operatorname{conv}(Q)$ is nonempty, convex and compact.
    For upper hemicontinuity of $\LMOp$, by \citet[Theorem 3.2]{beavis1990optimization}, for compact-valued correspondences it suffices to show the graph is closed. 
    Let $(\bm{p}^n,\bm{\mu}^n)\to(\bm{p},\bm{\mu})$ and let $\bm q^n\in\LMOp(\bm{p}^n,\bm{\mu}^n)$. By compactness of $\Delta_K$, extract a subsequence (not relabeled) with $\bm q^n\to\bm q$.
    By Carath\'{e}odory's theorem, we can represent $\bm{q}^n$ as $\bm{q}^n=\sum_{i=1}^{K}\lambda^n_i\bm{q}^n_i$
    with
    $\bm q^{n}_{i}\in
     \argmax_{\bm q\in \Delta_{K}}\bm q^{\top}\bm f^{n}_{i}$ and
    $\bm f^{n}_{i}\in\clarke F(\bm{p}^{n},\bm{\mu}^{n})$.
    Pass to a further subsequence so that
    $\lambda^{n}_{i}\to\lambda_{i}$,
    $\bm q^{n}_{i}\to\bm q_{i}$, and
    $\bm f^{n}_{i}\to\bm f_{i}$.
    By closedness of the Clarke gradient, we have $\bm f_{i}\in\clarke F(\bm{p},\bm{\mu})$, and 
    $
      \bm q_{i}^{\top}\bm f_{i}
      =\max_{\bm q'\in \Delta_{K}}\bm{q}'^{\top}\bm f_{i}.
    $
    Hence $\bm q=\sum_{i=1}^{K}\lambda_{i}\bm q_{i}\in\LMOp(\bm{p},\bm{\mu})$, so the graph of
    $\LMOp$ is closed.
\end{proof}

\begin{proof}[Proof of Theorem~\ref{thm:convergence}]
Fix any solution $(\bm{p}(t),\bm{\mu}(t))$ and measurable selection $(\bm{q}(t),\bm{\nu}(t))$ guaranteed by Lemma~\ref{lem:existence_solution}. We now consider the value of the Lyapunov function along the solution:
\begin{align*}
    V(t) = V(\bm{p}(t),\bm{\mu}(t)) 
        & = \max_{\bm{q}' \in \Delta_{K}} (\bm{q}' - \bm{p}(t))^{\top} \nabla_{\bm{p}} F(\bm{p}(t),\bm{\mu}(t)) - \min_{\bm{\nu}' \in \Delta_{|\mathcal{X}|}} (\bm{\nu}' -\bm{\mu}(t))^{\top} \nabla_{\bm{\mu}} F(\bm{p}(t),\bm{\mu}(t)) \\
        & = (\bm{q}(t)-\bm{p}(t))^{\top} \nabla_{\bm{p}} F(\bm{p}(t),\bm{\mu}(t)) - (\bm{\nu}(t)-\bm{\mu}(t))^{\top} \nabla_{\bm{\mu}} F(\bm{p}(t),\bm{\mu}(t)) \\
        & = \dot{\bm{p}}(t)^{\top} \nabla_{\bm{p}} F(\bm{p}(t),\bm{\mu}(t)) - \dot{\bm{\mu}}(t)^{\top} \nabla_{\bm{\mu}} F(\bm{p}(t),\bm{\mu}(t)).
\end{align*}
By assumption, $V(t)$ is differentiable for almost every $t \ge 0$, and
\begin{align*}
    \frac{\diff}{\diff t}\Bigl((\bm{q} - \bm{p})^{\top} \nabla_{\bm{p}} F\Bigr) 
        & = \frac{\diff}{\diff t}\Bigl(\bm{q}^{\top} \nabla_{\bm{p}} F\Bigr) - \frac{\diff}{\diff t}\Bigl(\bm{p}^{\top} \nabla_{\bm{p}} F\Bigr) 
        = \bm{q}^{\top} \frac{\diff}{\diff t}\Bigl(\nabla_{\bm{p}} F\Bigr) - \Bigl[\dot{\bm{p}}^{\top} \nabla_{\bm{p}} F + \bm{p}^{\top} \frac{\diff}{\diff t}\Bigl(\nabla_{\bm{p}} F\Bigr)\Bigr] \\
        & = (\bm{q} - \bm{p})^{\top} \frac{\diff}{\diff t}\Bigl(\nabla_{\bm{p}} F\Bigr) - \dot{\bm{p}}^{\top} \nabla_{\bm{p}} F 
        = - \dot{\bm{p}}^{\top} \nabla_{\bm{p}} F + \dot{\bm{p}}^{\top} \nabla_{\bm{pp}}F \dot{\bm{p}} + \dot{\bm{p}}^{\top} \nabla_{\bm{p\mu}}F \dot{\bm{\mu}},
\end{align*}
where the second equality follows from The Envelope theorem, \citet[Theorem 1]{milgrom2002envelope}. 
Similarly, we obtain
$
    \frac{\diff}{\diff t}\bigl((\bm{\nu} - \bm{\mu})^{\top} \nabla_{\bm{\mu}} F\bigr) = - \dot{\bm{\mu}}^{\top} \nabla_{\bm{\mu}} F + \dot{\bm{\mu}}^{\top} \nabla_{\bm{\mu \mu}}F \dot{\bm{\mu}} + \dot{\bm{\mu}}^{\top} \nabla_{\bm{\mu p}}F \dot{\bm{p}}.
$
Hence, 
\begin{align*}
    \frac{\diff}{\diff t}V 
        & = \frac{\diff}{\diff t}\Bigl((\bm{q} - \bm{p})^{\top} \nabla_{\bm{p}} F\Bigr) - \frac{\diff}{\diff t}\Bigl((\bm{\nu} - \bm{\mu})^{\top} \nabla_{\bm{\mu}} F\Bigr) \\
        & = - \Bigl(\dot{\bm{p}}^{\top} \nabla_{\bm{p}} F - \dot{\bm{\mu}}^{\top} \nabla_{\bm{\mu}} F\Bigr)  + \dot{\bm{p}}^{\top} \nabla_{\bm{pp}}F \dot{\bm{p}} + \dot{\bm{p}}^{\top} \nabla_{\bm{p\mu}}F \dot{\bm{\mu}} - \dot{\bm{\mu}}^{\top} \nabla_{\bm{\mu \mu}}F \dot{\bm{\mu}} - \dot{\bm{\mu}}^{\top} \nabla_{\bm{\mu p}}F \dot{\bm{p}} \\
        & = -V + \dot{\bm{p}}^{\top} \nabla_{\bm{pp}}F \dot{\bm{p}} - \dot{\bm{\mu}}^{\top} \nabla_{\bm{\mu \mu}}F \dot{\bm{\mu}}
         \le -V,
\end{align*}
where the symmetry $\nabla_{\bm{p\mu}}F = (\nabla_{\bm{\mu p}}F)^{\top}$ is used along with the definition of $V$, and the last inequality follows Since $F$ is concave--convex.
Thus, for almost every $t \ge 0$, we have $\frac{\diff}{\diff t}V(t) \le -V(t)$. Since $V(\cdot)$ is absolutely continuous, applying Gr\"{o}nwall's inequality yields $ V(t) \le V(0)e^{-t}$.
\end{proof}

\begin{proof}[Proof of Lemma~\ref{lem:gap}]
Concavity in $\bm{p}$ implies 
$
F(\bm q,\bm{\mu}) - F(\bm{p},\bm{\mu})
 \le  (\bm q-\bm{p})^\top\nabla_{\bm{p}}F(\bm{p},\bm{\mu})
 \le  \max_{\bm q'\in\Delta_K} (\bm q'-\bm{p})^\top\nabla_{\bm{p}}F(\bm{p},\bm{\mu})
$ for any $\bm q\in\Delta_K$.
Convexity in $\bm{\mu}$ implies 
$
F(\bm{p},\bm{\mu}) - F(\bm{p},\bm\nu)
 \le (\bm\nu-\bm{\mu})^\top\nabla_{\bm{\mu}}F(\bm{p},\bm{\mu})
 \le  \max_{\bm\nu'\in\Delta_{|\mathcal X|}} (\bm\nu'-\bm{\mu})^\top\nabla_{\bm{\mu}}F(\bm{p},\bm{\mu})
$ for any $\bm\nu\in\Delta_{|\mathcal X|}$.
Taking the maximum over $\bm q$ and the minimum over $\bm\nu$ in the primal--dual gap definition, and combining the above bounds, the proof is complete.
\end{proof}

\begin{proof}[Proof of Theorem~\ref{thm:value_convergence}]
By Lemma~\ref{lem:strict_positivity}, the solution remains in the interior of the simplices for any finite $t\ge 0$, so that all gradients are well-defined by Assumption~\ref{assumption:regularity}. 
By Lemma~\ref{lem:gap}, $\mathrm{Gap}(\bm{p},\bm{\mu})  \le  V(\bm{p},\bm{\mu})$ for all interior points.
By Theorem~\ref{thm:convergence}, the Lyapunov function $V(\bm{p}(t),\bm{\mu}(t))$ satisfies
$
V(\bm{p}(t),\bm{\mu}(t)) \le V(\bm{p}(0),\bm{\mu}(0)) e^{-t}.
$
Thus, 
$
\lim_{t\to\infty} \mathrm{Gap}(\bm{p}(t),\bm{\mu}(t)) \le \lim_{t\to\infty} V(\bm{p}(t),\bm{\mu}(t)) = 0.
$
By the definition of the duality gap, we obtain
$
\max_{\bm{q}\in\Delta_{K}}F(\bm{q},\bm{\mu}(t)) - \min_{\bm{\nu}\in\Delta_{|\mathcal{X}|}}F(\bm{p}(t),\bm{\nu}) \to 0 
$
as $t\to\infty$.
This implies that for every $\epsilon>0$, there exists $T = T(\epsilon) > 0$ such that for all $t\ge T$,
$
\max_{\bm{q}\in\Delta_{K}}F(\bm{q},\bm{\mu}(t)) - F(\bm{p}(t),\bm{\mu}(t)) < \epsilon 
$ and $F(\bm{p}(t),\bm{\mu}(t)) - \min_{\bm{\nu}\in\Delta_{|\mathcal{X}|}}F(\bm{p}(t),\bm{\nu}) < \epsilon.
$
Hence, for $t\ge T$, we have $\max_{\bm{q}\in\Delta_{K}}F(\bm{q},\bm{\mu}(t)) - \epsilon < F(\bm{p}(t),\bm{\mu}(t)) < \min_{\bm{\nu}\in\Delta_{|\mathcal{X}|}}F(\bm{p}(t),\bm{\nu}) + \epsilon.
$
Since the saddle-point value $F^*$ satisfies
$
\max_{\bm{q}\in\Delta_{K}}F(\bm{q},\bm{\mu}(t)) \ge F^* \ge \min_{\bm{\nu}\in\Delta_{|\mathcal{X}|}}F(\bm{p}(t),\bm{\nu}),
$
it follows that for every $\epsilon>0$, there exists $T = T(\epsilon) > 0$ such that for all $t\ge T$, we have $F^* - \epsilon < F(\bm{p}(t),\bm{\mu}(t)) < F^* + \epsilon.$
Therefore, the objective value converges to the optimal game value.
\end{proof}

\begin{proof}[Proof of Proposition~\ref{prop:interpolation}]
Fix $t\in [s,s+T]$ satisfying $t\in(t_k,t_{k+1})$. Here $k\ge n$ since $s\ge t_n$. 
By \eqref{def:hat_p_mu} and $t_{k+1}-t_k=1/(k+1)$, we have
$
\frac{\diff}{\diff t}\bigl(\hat{\bm{p}}(t),\hat{\bm{\mu}}(t)\bigr)
=\Bigl(\frac{\bm{p}_{k+1}-\bm{p}_k}{1/(k+1)}, \frac{\bm{\mu}_{k+1}-\bm{\mu}_k}{1/(k+1)}\Bigr)
=(\bm q_k-\bm{p}_k, \bm\nu_k-\bm{\mu}_k)\in G(\bm{p}_k,\bm{\mu}_k).
$
Moreover, for any $r\in[0,1/(k+1)]$, we have
$
\|(\hat{\bm{p}}(t_k+r),\hat{\bm{\mu}}(t_k+r))-(\bm{p}_k,\bm{\mu}_k)\|_1
=\frac{r}{1/(k+1)}\|(\bm{p}_{k+1}-\bm{p}_k,\bm{\mu}_{k+1}-\bm{\mu}_k)\|_1
\le \|(\bm{p}_{k+1}-\bm{p}_k,\bm{\mu}_{k+1}-\bm{\mu}_k)\|_1 =\frac{1}{k+1}\|(\bm q_k-\bm{p}_k,\bm\nu_k-\bm{\mu}_k)\|_1\le \frac{4}{k+1} \le \frac{4}{n+1},
$
where we used the fact that $\| \bm q_k-\bm{p}_k\|_1\le 2$ and $\|\bm\nu_k-\bm{\mu}_k\|_1\le 2$.
Thus, with $(\bm{p}',\bm{\mu}')=(\bm{p}_k,\bm{\mu}_k)$ we have
\[
\|(\hat{\bm{p}}(t),\hat{\bm{\mu}}(t))-(\bm{p}',\bm{\mu}')\|_1
+ d\!\left(\tfrac{\diff}{\diff t}(\hat{\bm{p}}(t),\hat{\bm{\mu}}(t)),  G(\bm{p}',\bm{\mu}')\right)
\le \tfrac{4}{k+1} \le \tfrac{4}{n+1}, \quad \text{for all } t \in (t_k,t_{k+1}).
\]
Hence $\frac{\diff}{\diff t}(\hat{\bm{p}},\hat{\bm{\mu}})\in G^{\frac{4}{n+1},T}(\hat{\bm{p}},\hat{\bm{\mu}})$ a.e. on $[s,s+T]$, and the claim follows by definition of $\mathcal T^{\frac{4}{n+1},T}$.
\end{proof}

\section{Conclusion}

We proposed Frank--Wolfe Self-Play, a minimal modification of Frank--Wolfe that solves the mixed-strategy saddle-point formulation of pure exploration problems in structured bandits.
Conceptually, the mixed-strategy skeptic is the right abstraction: it converts asymptotic problem complexity into a clean concave--convex game and aligns algorithmic updates with KKT complementary slackness.
Algorithmically, FWSP is appealingly simple, using one-hot steps on simplices and remains faithful to the ``sample one arm at a time'' bandit paradigm, without needing projections, regularization, or tuning parameters.

We established a Lyapunov function for the continuous-time limit that decreases exponentially, yielding vanishing duality gap and convergence to optimal game value. 
A careful stochastic-approximation argument via $\varepsilon$-perturbed differential inclusions then transfers these guarantees to the discrete updates of FWSP. 
Our linear-bandit case study exposes structural quirks (nonunique optima, bilinear objectives, and optimal designs with zero weight on the best arm) and demonstrates where popular top-two/Thompson-style methods can be systematically suboptimal, while FWSP still finds the correct allocation.

Several directions for future work are compelling.
First, rates of convergence for the discrete algorithm and last-iterate behavior under nonunique equilibria are open even for the bilinear case \cite{gidel2017frank,chen2024last}. 
Second, a full analysis of the learning variant of the algorithm would close the loop to practice. 
Third, FWSP can be readily extended to pure exploration with multiple correct answers \cite{degenne2019pure} and to cost-aware exploration \cite{qin2024optimizing}.
Finally, our Lyapunov-and-DI toolkit suggests a path to analyze other projection-free bandit procedures, potentially unifying algorithm design for pure exploration under a common game-theoretic and dynamical-systems framework.

\bibliography{refs}
\bibliographystyle{plainnat}

\clearpage

\appendix
\part*{Appendix}

\section{Further Related Work}\label{app:literature}
\paragraph{Pure Exploration.}
Asymptotically tight fixed-confidence lower bounds for unstructured BAI and the Track-and-Stop (TaS) algorithm were given by \citet{garivier2016optimal}, but the latter solves a new optimization at every step. 
The lower bound has since then been extended for a variety of pure exploration tasks
\citep{soare2014best,degenne2020gamification,shen2021ranking,du2024contextual,zhou2025linear}.
To reduce computation of TaS, \citet{menard2019gradient} use subgradient ascent, and \citet{wang2021fast} propose Frank--Wolfe Sampling (FWS) for the nonsmooth objective. 
Game-theoretic approaches further refine the picture \citep{degenne2019non,degenne2020gamification,zaki2022improved}, but hinge on regret guarantees that can be delicate (e.g., fictitious play can incur linear regret under adversarial tie-breaking; \citealp{kalai2005efficient}). Methodologically, FWS is closest to our approach but it requires constructing a subdifferential subspace and solving a linear program.

\paragraph{Game Theory.}
Our work is also closely connected to evolutionary game theory.
When the objective function is bilinear, the maximin problem reduces to a two-player zero-sum matrix game, and our algorithm coincides with fictitious play (FP) \citep{brown1951iterative}.
\citet{robinson1951iterative} proved its convergence in two-player zero-sum matrix games, while \citet{shapiro1958note} observed that their proof implied a convergence rate of $O\bigl(n^{-1/(p+q-2)}\bigr)$ where $p$ and $q$ denote the dimensions of the rows and columns of the payoff matrix, respectively. 
\citet{karlin1959mathematical} conjectured a faster $O(t^{-1/2})$; but \citet{daskalakis2014counter} later provided a counterexample, showing that Shapiro's bound is in fact tight under adversarial tie-breaking.
More recently, \citet{abernethy2021fast} established an $O(t^{-1/2})$ rate for diagonal matrix games under fixed lexicographic tie-breaking, and \citet{lazarsfeld2025fast} proved the same rate for Rock-Paper-Scissors games under arbitrary tie-breaking.
Another line of work analyzes FP and best-response dynamics (BRD) via Lyapunov methods: \citet{hofbauer1995stability,harris1998rate} first studied continuous-time BRD and later discrete FP in matrix games, with \citet{hofbauer2006best} extending the analysis to concave--convex zero-sum games.
In the optimization literature, \citet{gidel2017frank} introduced Saddle-Point Frank--Wolfe and proved an $O(t^{-1/2})$ rate in \emph{strongly convex-concave} games, while \citet{chen2024last} obtained the same rate for a \emph{smoothed variant} of SPFW in the setting of monotone variational inequalities.
By contrast, our objective is neither bilinear nor strongly concave--convex, and it is moreover nonsmooth on the boundary of the simplex.
Thus, existing convergence analyses do not apply, motivating the new technical arguments we develop.
Notably, prior studies report promising empirical performance \citep{kaufmann2020contributions,degenne2020gamification}, although theoretical guarantees were unavailable before this work.

\paragraph{Minimax Optimization.}
No-regret learning provides a unifying lens for iterative game dynamics: sublinear regret (typically $O(\sqrt{T})$) ensures time-averaged strategies approach Nash equilibria in zero-sum games \citep{cesa2006prediction,hazan2016introduction}.
Classical algorithms such as Hedge and Follow-the-Regularized-Leader (FTRL) thus yield $O(1/\sqrt{T})$ convergence to an optimal mixed strategy \citep{freund1999adaptive}.
Beyond this baseline, \citet{daskalakis2011near} achieved $O(\log T/T)$ for saddle-point problems via the excessive-gap technique, and \citet{rakhlin2013optimization} matched this rate with Optimistic Hedge by canceling oscillatory terms in gradient dynamics.
Designing faster coordinating learners remains active \citep{abernethy2018faster,WeiLZL21,lin2020finite,daskalakis2021near,cai2022tight,cen2024fast}.
We complement this line by introducing two simple greedy learners for the experimenter and the skeptic, and proving convergence.

\paragraph{Stochastic Approximation.}
Our technical tool is closest to the ODE method for stochastic approximation \citep{ljung1977analysis}, developed for dynamical systems by \citet{benaim1999dynamics} and \citet{borkar2000ode} and extended to differential inclusions by \citet{benaim2005stochastic}.
These tools are widely used in simulation \citep{hu2022stochastic,he2024adaptive,hu2025quantile} and reinforcement learning \citep{tsitsiklis1996analysis,rowland2024analysis}, but, to our knowledge, are new to adaptive experimentation.
Recent works use discrete-time Lyapunov functions to obtain finite-sample guarantees \citep{srikant2019finite,chen2024lyapunov,zhang2025piecewise}; adapting such techniques to our setting is an open direction.
While convergence rates for the ODE method are well understood \citep{Kushner1997StochasticAA,borkar2008stochastic}, the picture is far less complete for differential inclusions; the only available result \citep{Nguyen2021StochasticAW} assumes conditions that our algorithm does not satisfy.

\section{Numerical Example}\label{sec:num}
In this section, we present simple numerical examples to validate the theoretical convergence guarantees of both the nonlearning version of FWSP (with closed-form updates in Remark~\ref{rmk:closed_form}) and the posterior-sampling-based learning version described in Algorithm~\ref{alg:learn}.
We consider two cases:
\begin{itemize}
    \item \textbf{Case 1}: The linear bandit example in Example~\ref{ex:TS_fail}, where 
    \[\bm{p}^* = (0,2/3,1/3), \quad \bm{\mu}^* = (1,0), \quad F^* = 2/9.\]
    \item \textbf{Case 2}: A randomly generated linear bandit instance with $K = 6$, $D = 3$, and
    \[\bm{\theta} = (1.9492, -0.4601, -0.4279),\]
    \[A = \begin{bmatrix}
        \bm{a}_1^{\top} \\
        \bm{a}_2^{\top} \\
        \bm{a}_3^{\top} \\
        \bm{a}_4^{\top} \\
        \bm{a}_5^{\top} \\ 
        \bm{a}_6^{\top} \\
    \end{bmatrix}
    =
    \begin{bmatrix}
          -0.7322,& -0.7272,& -0.0976 \\
          -0.9580,& -0.2982,&  0.8227\\
          -0.0585,& -0.8511,&  0.1397\\
           0.2705,& -0.8211,&  0.1124\\
           0.5793,& -0.5567,& -0.1627\\
          -0.5004,& -0.4163,&  0.6065
    \end{bmatrix}\]
    The mean vector is $\bm{m} = (-1.0509, -2.0821,  0.2178,  0.8569,  1.4549, -1.0434)$ and the optimal arm is $I^* = 5$.
    The optimal game value is $F^* = 0.5037$ with unique optimal solution:
    \begin{align*}
        \bm{p}^* & = (0.3122, 0.3856, 0, 0.3022, 0, 0),\\
        \bm{\mu}^* & = (\mu^*_1,\mu^*_2,\mu^*_3,\mu^*_4,\mu^*_6) = (0.5149, 0, 0, 0.4851, 0).
    \end{align*}
\end{itemize}
Note that this is yet another example where $p^*_{I^*} = p^*_5 = 0$.

For each instance, we run FWSP (without learning) for $10^7$ rounds and record the Lyapunov value $V_t \triangleq V(\bm{p}_t, \bm{\mu}_t)$, computed using the ground-truth parameter $\bm{\theta}$.
Since $V$ upper-bounds the duality gap, a decreasing $V_t$ indicates convergence of the game value. We likewise run Algorithm~\ref{alg:learn} for $10^7$ rounds and record $V_t$. 
We repeat this procedure for 100 independent replications and report the mean trajectory together with the first and third quartiles. 
The results for the two examples are shown in Figures~\ref{fig:Ex_classic} and~\ref{fig:Ex_rand}, respectively.
In both cases, the Lyapunov trajectory under FWSP drops rapidly, consistent with Theorem~\ref{thm:discrete_convergence}. 
The learning-based variant also converges, though at a slower pace in Case 2.

\begin{figure}[hbtp]
    \centering
    \includegraphics[width=0.7\linewidth]{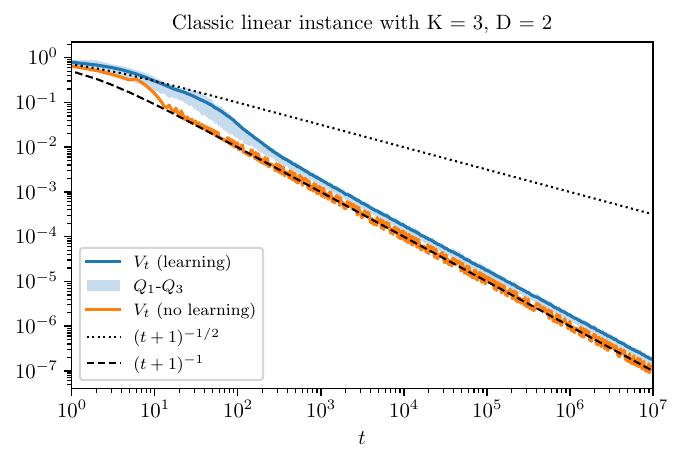}
    \caption{The trajectory of the Lyapunov function $V_t$ for \textbf{Case 1}.}
    \label{fig:Ex_classic}
\end{figure}

\begin{figure}[hbtp]
    \centering
    \includegraphics[width=0.7\linewidth]{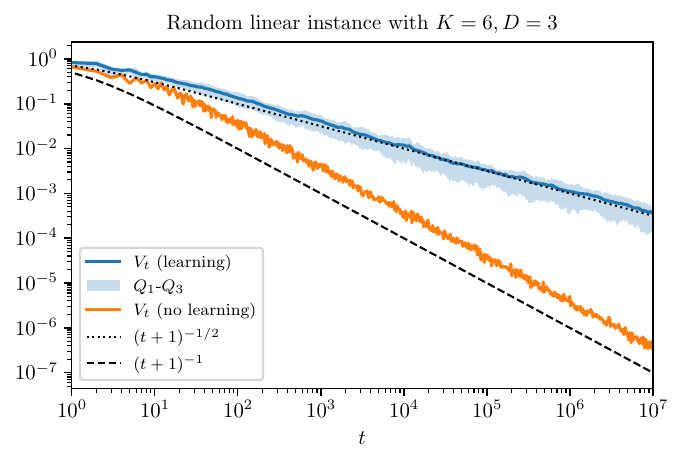}
    \caption{The trajectory of the Lyapunov function $V_t$ for \textbf{Case 2}.}
    \label{fig:Ex_rand}
\end{figure}

\section{Auxiliary facts}\label{sec:facts}
In this section, we compile several classical results on set-valued maps and differential inclusions that facilitate our proofs.

\paragraph{Envelope theorem.}
Let $X$ be a choice set and let $t \in [0,1]$ be the relevant parameter. Consider the parameterized objective function $f: X \times [0,1] \to \mathbb{R},$
and define the value function $v: [0,1] \to \mathbb{R}$ and the optimal choice correspondence (set-valued function) $X^*$ by
\[
v(t) = \sup_{x \in X} f(x,t) \quad \text{and} \quad X^*(t) = \{x \in X : f(x,t) = v(t)\}.
\]

\begin{theorem}[Envelope theorem, Theorem 1, \citealt{milgrom2002envelope}]\label{thm:envelope}
Assume that for a given $t \in [0,1]$ and for some $x^* \in X^*(t)$ the partial derivative $f_t(x^*,t)$ exists. If $v$ is differentiable at $t$, then $v'(t) = f_t(x^*,t).$
\end{theorem}

\paragraph{Upper hemicontinuity for compact-valued correspondence.}

\begin{theorem}[Theorem 3.2, \citealt{beavis1990optimization}]
\label{thm:compact_hemicontinuous}
The compact-valued set-valued map $F: X \rightrightarrows Y$ is upper hemicontinuous at $x \in X$ if, and only if, for every sequence $\{x_n\}$ converging to $x$ and every sequence $\{y_n\}$ with $y_n \in F(x_n)$, there exists a converging subsequence of $\{y_n\}$ whose limit belongs to $F(x)$.
\end{theorem}

\paragraph{Berge's Maximum Theorem.}
\begin{theorem}[Maximum Theorem]
Let $X\subset \mathbb{R}^m$, $Y \subset \mathbb{R}^k$ and $\mathcal{X}i: X \rightrightarrows Y$ be a set-valued map with nonempty, compact values. Let $f: X \times Y \to \mathbb{R}$ be a continuous function. Define the set-valued function $M: X\rightrightarrows Y$, the maximizers $M(x) \triangleq \argmax_{y\in\mathcal{X}i(x)} f(x,y)$,
and the corresponding value function $v: X \to \mathbb{R}$ by $v(x) = \max_{y\in\mathcal{X}i(x)} f(x,y)$. If $\mathcal{X}i$ is continuous at $x$, then  $v$ is continuous at $x$ and the set-valued function $M$ is closed, compact-valued and upper hemicontinuous at $x$. 
\end{theorem}

In our context, we take $x = (\bm{p},\bm{\mu})\in \interior{\Delta_K} \times \Delta_{|\mathcal{X}|}$,  $y = (\bm{q},\bm{\nu})$, and $X = Y = \Delta_{K} \times \Delta_{|\mathcal{X}|}$.
Let the set-valued map $\mathcal{X}i(\bm{p},\bm{\mu}) = \Delta_{K} \times \Delta_{|\mathcal{X}|}$ for all $(\bm{p},\bm{\mu})$, 
which is constant (and hence continuous) at any $x$.
Let $f(\bm{p},\bm{\mu};\bm{q},\bm{\nu}) = (\bm{q}^{\top}\nabla_{\bm{p}} F(\bm{p},\bm{\mu}),\bm{\nu}^{\top}\nabla_{\bm{\mu}} F(\bm{p},\bm{\mu}))$.
By part 1 of Assumption~\ref{assumption:regularity}, $\nabla_{\bm{p}}F$ and $\nabla_{\bm{\mu}}F$ are continuous.
The maximizer set-valued function $M(\bm{p},\bm{\mu}) = \LMOp(\bm{p},\bm{\mu})\times \LMOmu(\bm{p},\bm{\mu})$. Consequently, both LMO correspondences are upper hemicontinuous.

\paragraph{Existence of solution to differential inclusion.}

\begin{definition}[Viable trajectory]
    The trajectory $\bm{x}(\cdot)$ such that 
    \[\forall t\in[0,T),\quad\bm{x}(t)\in K\]
    is called a viable trajectory on $[0,T)$.
\end{definition}

\begin{definition}[Contingent cone]
    We define the contingent cone $T_{K}(\bm{x})$ as
    \begin{align*}
        T_{K}(\bm{x})&\triangleq\{\bm{v}: \exists\text{ strictly positive sequence } h_n\to 0 \\
        & \qquad \qquad\text{ and elements }\bm{u}_n\to\bm{v} \text{ such that }\forall n\geq 0, \bm{x}+h_n\bm{u}_n\in K\}.
    \end{align*}
\end{definition}

\begin{theorem}[Viability Theorem, Theorem 1, Section 4.2, \citealt{aubin_differential_1984}]\label{thm:existence_viability}
Let $K$ be a subset of Hilbert space $X$ and $T_{K}(\bm{x})$ be $K$'s contingent cone at $\bm{x}$. Consider upper hemicontinuous map $F$ from $K$ to $X$ with compact convex values. We posit the tangential condition:
\[\forall \bm{x}\in K,\quad F(x)\cap T_{K}(\bm{x})\neq\varnothing.\]
When $X$ is finite dimensional and $F(K)$ is bounded, then for all $\bm{x}_0\in K$, there exists a viable trajectory $x: [0,\infty) \to X$ that is a solution of the differential inclusion
\[
\dot{x}(t) \in F(t,x(t)) , \quad x(0)=x_0.
\]
\end{theorem}

In our context, let $K \triangleq \interior{\Delta_K}\times\Delta_{|\mathcal{X}|}$ and $F\triangleq (\LMOp(\bm{p},\bm{\mu})-\bm{p},\LMOmu(\bm{p},\bm{\mu})-\bm{\mu})$.
The tangential condition is satisfied.
To see this, note that $\LMOp(\bm{p},\bm{\mu})\subset\Delta_K$ and $\LMOmu(\bm{p},\bm{\mu})\subset\Delta_{|\mathcal{X}|}$, for any $(\bm{p},\bm{\mu})\in\interior{\Delta_K}\times\Delta_{|\mathcal{X}|}$ and any $\bm{v}\in F(\bm{p},\bm{\mu})$ we have $(\bm{p},\bm{\mu})+\epsilon \bm{v}\in \interior{\Delta_K}\times\Delta_{|\mathcal{X}|}$ for all sufficiently small $\epsilon>0$, hence $F(\bm{p},\bm{\mu})\subseteq T_{K}(\bm{p},\bm{\mu})$.

\paragraph{Existence of measurable selections.}

\begin{theorem}[Corollary 1, Section 1.14, \citealt{aubin_differential_1984}]
Let $ f : X \times U \to X $ be continuous, where $ U $ is a compact separable metric space. Assume that there exist an interval $ I $ and an absolutely continuous function $x: I \to \mathbb{R}^n,$ such that 
\[
x'(t) \in f(x(t), U) \quad \text{for almost every } t \in I.
\]
Then, there exists a Lebesgue measurable function $u: I \to U $
such that 
\[
x'(t) = f(x(t), u(t)) \quad \text{for almost every } t \in I.
\]
\end{theorem}

In our context, let $U = \Delta_{|\mathcal{X}|} \times \Delta_{K}$ denote the feasible set for the FW-responses.
And the interval $I$ can be selected as $I = [0,\infty)$ by the existence of absolute continuous solution in Lemma~\ref{lem:existence_solution}.

\paragraph{Integral representation for upper hemicontinuous map with compact convex images.}
\begin{theorem}[Lemma 1, Section 2.1, \citealt{aubin_differential_1984}]\label{thm:integral}
Let $F$ be an upper semicontinuous map from $I \times X$ into the compact convex subsets of $X$. Then the continuous function $x(\cdot)$ is a solution on $I$ to the inclusion
\[x'(t) \in F(t, x(t))\]
if and only if for every pair $(t_1, t_2)$,
\[x(t_2) \in x(t_1) + \int_{t_1}^{t_2} F(s, x(s))   ds.\]
\end{theorem}

\paragraph{Gr\"{o}nwall's inequality.}
\begin{theorem}[Gr\"{o}nwall's Inequality]
Let $\beta(\cdot)$ and $u(\cdot)$ be real-valued continuous functions defined on $[0,\infty)$. Suppose that $u$ is absolutely continuous and satisfies the differential inequality
$$
u'(t) \le \beta(t) u(t), \quad \text{for almost all } t \ge 0.
$$
Then, we have
$$
u(t) \le u(0) \exp\left( \int_0^t \beta(s) \diff s \right).
$$
\end{theorem}

In our context, we can show that the candidate Lyapunov function $V(t) = V(\bm{p}(t),\bm{\mu}(t))$ is Lipschitz and hence absolutely continuous. We then take $\beta(t) = -1$ and $u(t) = V(t)$.

\paragraph{Compactness theorem.}
\begin{theorem}[Theorem 4, Section 0.4, \citealt{aubin_differential_1984}]\label{thm:compactness}
Consider a sequence of absolutely continuous functions $x_k(\cdot)$ from an interval $I$ of $\mathbb{R}$ to a Banach space $X$ satisfying
\begin{itemize}
\item[(i)] $\forall t \in I$, $\{ x_k(t) \}_k$ is a relatively compact subset of $X$,
\item[(ii)] there exists a positive function $c(\cdot) \in \mathcal{L}^1(I, x)$ such that, for almost all $t \in I$, $\| \dot{x}_k(t) \| \leq c(t).$
\end{itemize}
Then there exists a subsequence (again denoted by) $x_k(\cdot)$ converging to an absolutely continuous function $x(\cdot)$ from $I$ to $X$ in the sense that
\begin{itemize}
    \item[(i)] $ x_k(\cdot) $ converges uniformly to $ x(\cdot) $ over compact subsets of $ I $,
    \item[(ii)] $ \dot{x}_k(\cdot) $ converges weakly to $ \dot{x}(\cdot) $ in $ \mathcal{L}^1(I, X) $.
\end{itemize}
\end{theorem}

\paragraph{Convergence theorem: Property of upper hemicontinuous correspondences with closed values.}
\begin{theorem}[Theorem 1, Section 1.4, \citealt{aubin_differential_1984}]\label{thm:set-valued map continuous}
Let $F$ be a proper hemicontinuous map from a Hausdorff locally convex space $X$ to the closed convex subsets of a Banach space $Y$. Let $I$ be an interval of $\mathbb{R}$ and $x_k(\cdot)$ and $y_k(\cdot)$ be measurable functions from $I$ to $X$ and $Y$ respectively satisfying
for almost all $t \in I$, for every open neighborhood $\mathcal{N}$ of $0\in X \times Y$, there exists $k_0 \doteq k_0(t, \mathcal{N})$ such that
\begin{align}
    \forall k \geq k_0, \quad (x_k(t), y_k(t)) \in \mathrm{graph}(F) + \mathcal{N}
\end{align}
If $x_k(\cdot)$ converges almost everywhere to a function $x(\cdot)$ from $I$ to $X$, $y_k(\cdot)\in \mathcal{L}^1(I, Y)$ and converges weakly to $y(\cdot)\in \mathcal{L}^1(I, Y)$,
then for almost all $t \in I$,
\[
(x(t), y(t)) \in \mathrm{graph}(F), \quad \text{i.e., } y(t) \in F(x(t)).
\]    
\end{theorem}

\end{document}